\documentclass{article}

\usepackage[utf8]{inputenc} 
\usepackage[T1]{fontenc}   
\usepackage{url}            %
\usepackage{booktabs}       %
\usepackage{amsfonts}       %
\usepackage{nicefrac}       %
\usepackage{microtype}      %
\usepackage{xcolor}         %
\usepackage{mathtools}
\usepackage{colortbl}
\usepackage[top=1.3in, left=1.2in, bottom=1.3in, right=1.2in]{geometry}

\usepackage{xcolor}
\definecolor{CColor}{rgb}{0.01,0.31,0.59}
\definecolor{GGray}{rgb}{0.80,0.90,1}
\definecolor{Shady}{rgb}{0.9,0.9,0.9}
\definecolor{kaistblue}{RGB}{20,135,200}
\definecolor{kaistdarkblue}{RGB}{0,65,145}
\definecolor{urbanablue}{RGB}{19,41,75}
\definecolor{urbanaorange}{RGB}{232,74,39}
\definecolor{drp}{rgb}{0.53,0.15,0.34}
\usepackage[plainpages=false,pdfpagelabels,colorlinks=true,linkcolor=CColor,citecolor=CColor]{hyperref}
 
\usepackage[T1]{fontenc}
\usepackage{lmodern}

\usepackage{graphicx}
\usepackage{amsmath}
\usepackage{amssymb}
\usepackage{mathtools}
\usepackage{amsthm}
\usepackage[capitalize,noabbrev]{cleveref}
\usepackage[textsize=tiny]{todonotes}
\usepackage{booktabs}
\usepackage{subcaption}
\usepackage[round]{natbib}

\input{packages}
\newcommand{\mathboldcommand}[1]{\mathbb{#1}}

\newcommand{\bbN}{\mathboldcommand{N}}

\newcommand{\bbR}{\mathboldcommand{R}}

\usepackage{euscript}
\newcommand{\mathcalcommand}[1]{\mathcal{#1}}
\newcommand{\mcA}{\mathcalcommand{A}}
\newcommand{\mcB}{\mathcalcommand{B}}
\newcommand{\mcC}{\mathcalcommand{C}}
\newcommand{\mcD}{\mathcalcommand{D}}

\newcommand{\mcG}{\mathcalcommand{G}}
\newcommand{\mcH}{\mathcalcommand{H}}
\newcommand{\mcI}{\mathcalcommand{I}}
\newcommand{\mcJ}{\mathcalcommand{J}}
\newcommand{\mcK}{\mathcalcommand{K}}

\newcommand{\mcS}{\mathcalcommand{S}}
\newcommand{\mcT}{\mathcalcommand{T}}
\newcommand{\mcU}{\mathcalcommand{U}}
\newcommand{\mcV}{\mathcalcommand{V}}

\newcommand{\mcX}{\mathcalcommand{X}}
\newcommand{\mcY}{\mathcalcommand{Y}}

\DeclareMathAlphabet{\mathpzc}{T1}{pzc}{m}{it}

\newcommand{\mathfrakcommand}[1]{\mathfrak{#1}}

\newcommand{\fkS}{\mathfrakcommand{S}}

\allowdisplaybreaks

\newcommand*{\commentout}[1]{}

\newlength{\parskiptrue}
\definecolor{lred}{rgb}{1.0, 0.5, 0.5}
\definecolor{lorange}{rgb}{1.00, 0.90, 0.20}
\definecolor{lgreen}{rgb}{0.35, 0.95, 0.35}
\definecolor{lime}{rgb}{0.9, 1.0, 0.6}
\definecolor{lblue}{rgb}{1.0, 0.85, 0.75}

\makeatletter

\newcommand*\wthelper[2]{%
        \hbox{\dimen@\accentfontxheight#1%
                \accentfontxheight#11.1\dimen@
                $\m@th#1\widetilde{#2}$%
                \accentfontxheight#1\dimen@
        }%
}
\newcommand*\accentfontxheight[1]{%
        \fontdimen5\ifx#1\displaystyle
                \textfont
        \else\ifx#1\textstyle
                \textfont
        \else\ifx#1\scriptstyle
                \scriptfont
        \else
                \scriptscriptfont
        \fi\fi\fi3
}

\newcommand*\whhelper[2]{%
        \hbox{\dimen@\accentfontxheight#1%
                \accentfontxheight#11.2\dimen@
                $\m@th#1\widehat{#2}$%
                \accentfontxheight#1\dimen@
        }%
}
\makeatother

\makeatletter
\newcommand{\oset}[3][0ex]{%
  \mathrel{\mathop{#3}\limits^{
    \vbox to#1{\kern-3\ex@
    \hbox{$\scriptstyle#2$}\vss}}}}
\makeatother

\newcommand*{\defeq}{\triangleq}

\newcommand*{\relu}{\textsc{ReLU}}
\newcommand*{\relul}{\textsc{ReLU\text{-}Like}}
\newcommand*{\step}{\textsc{Step}}
\newcommand*{\sigmoid}{\textsc{Sigmoid}}

\newcommand*{\swish}{\textsc{Swish}}
\newcommand*{\mish}{\textsc{Mish}}
\newcommand*{\elu}{\textsc{ELU}}
\newcommand*{\celu}{\textsc{CeLU}}
\newcommand*{\selu}{\textsc{SeLU}}
\newcommand*{\silu}{\textsc{SiLU}}
\newcommand*{\gelu}{\textsc{GELU}}
\newcommand*{\softplus}{\textsc{Softplus}}

\newcommand*{\lrelu}{\text{Leaky-}\textsc{ReLU}}

\newcommand*{\hswish}{\textsc{HardSwish}}

\makeatletter
\def\moverlay{\mathpalette\mov@rlay}
\def\mov@rlay#1#2{\leavevmode\vtop{%
   \baselineskip\z@skip \lineskiplimit-\maxdimen
   \ialign{\hfil$\m@th#1##$\hfil\cr#2\crcr}}}
\newcommand{\charfusion}[3][\mathord]{
    #1{\ifx#1\mathop\vphantom{#2}\fi
        \mathpalette\mov@rlay{#2\cr#3}
      }
    \ifx#1\mathop\expandafter\displaylimits\fi}
\makeatother

\newcommand*{\dist}[2]{\mathsf{dist}\left({#1},{#2}\right)}
\newcommand*{\diam}{{\mathsf{diam}}}

\setlist{noitemsep, topsep=5pt}

 \AtBeginDocument{%
    \crefname{section}{Section}{Sections}%
    \crefname{appendix}{Appendix}{Appendices}%
    \crefname{subsection}{Section}{Sections}%
    \crefname{figure}{Figure}{Figures}%
}

\theoremstyle{plain}
\newtheorem{theorem}{Theorem}

\newtheorem{lemma}[theorem]{Lemma}

\theoremstyle{definition}
\newtheorem{definition}{Definition}
\newtheorem{condition}{Condition}

\title{Minimum width for universal approximation\\ using squashable activation functions}
\author{Jonghyun Shin\thanks{
  Department of Mathematics Education,
  Korea University}
 \ \ \ \ \ \ 
 Namjun Kim\thanks{
  Department of Artificial Intelligence,
  Korea University}
 \ \ \ \ \ \ 
 Geonho Hwang\thanks{
  Department of Mathematical Sciences,
  GIST\\
  $~$\quad~\, correspondence to: \texttt{sejun.park000@gmail.com}}
 \ \ \ \ \ \ 
 Sejun Park$^\dagger$}
\date{}

\usepackage[skip=10pt plus1pt, indent=0pt]{parskip}

\begin{document}

\maketitle

\begin{abstract}
The exact minimum width that allows for universal approximation of unbounded-depth networks is known only for $\relu$ and its variants.
In this work, we study the minimum width of 
networks using general activation functions. Specifically, we focus on
\emph{squashable} functions that can approximate the identity function and binary step function by alternatively composing with affine transformations.
We show that for networks using a squashable activation function to universally approximate $L^p$ functions from $[0,1]^{d_x}$ to $\bbR^{d_y}$, the minimum width is $\max\{d_x,d_y,2\}$ unless $d_x=d_y=1$; the same bound holds for $d_x=d_y=1$ if the activation function is monotone.
We then provide sufficient conditions for squashability and show that all non-affine analytic functions and a class of piecewise functions are squashable, i.e., our minimum width result holds for those general classes of activation functions.

\end{abstract}

\section{Introduction}

Understanding what neural networks can or cannot do is a fundamental problem in deep learning theory.
The classical universal theorem states that two-layer networks can approximate any continuous function if an activation function is non-polynomial \citep{cybenko89,hornik89,leshno93,pinkus99}. Likewise, several studies on memorization show that neural networks can fit arbitrary finite training dataset \citep{baum88, huang98}.
These results guarantee the existence of networks that can perform tasks in various practical applications such as computer vision \citep{he2016}, natural language processing \citep{vaswani2017,brown2020}, and science \citep{jumper2021}.

The minimum size of networks that can universally approximate or memorize has also been studied. 
For example, classical results show that the minimum depth for both universal approximation and memorization is exactly two \citep{pinkus99,baum88}. 
The minimum number of parameters depends on the depth of networks. 
For universal approximation using $\relu$ networks, it is known that shallow wide architectures require more parameters than deep narrow ones \cite{yarotsky18}, where similar results are also known for memorization \cite{park21b,vardi21}.
While these results show the benefits of depth, they also
imply the existence of the minimum width enabling universal approximation and memorization.

There have been extensive research efforts to characterize such a minimum width.
The minimum width for memorization is constantly bounded (i.e., independent of the input dimension) since any finite set of inputs can be mapped to distinct scalar values by projecting them \citep{park21b}.
Intriguingly, the minimum width for universal approximation depends on the input dimension $d_x$ and the output dimension $d_y$. Several works have shown that the minimum width lies between $d_x$ and $d_x+d_y+\alpha$ where $\alpha\ge0$ is some constant depending on the activation function and target functions space; however, the exact minimum width is known only for approximating $L^p$ functions when the activation function is $\relu$ or its variants \cite{park21,cai23,kim24}.

\subsection{Related works}
\begin{table*}[t]
\begin{center}
\caption{A summary of known bounds on the minimum width for universal approximation. %
}

\label{table:summary}
\begin{tabular}{| c | c  c | c |} 
 \hline
 {\bf Reference} & {\bf Function class} & {\bf Activation} $\sigma$ & {\bf Upper\,/\,lower bounds} \\ 
 \hline\hline

 \citet{Lu17} & $L^1(\mathbb R^{d_x}, \mathbb R)$ & $\relu$ & $d_x + 1 \le w_{\min} \le d_x + 4$ \\
 \hline
 \citet{hanin17} & $C([0,1]^{d_x}, \mathbb R^{d_y})$ & $\relu$ & $d_x + 1 \le w_{\min} \le d_x + d_y$ \\
 \hline
 \citet{johnson19} & $C([0,1]^{d_x}, \mathbb R^{d_y})$ & uniformly conti.$^\|$ & $d_x + 1 \le w_{\min}$ \\
 \hline
 \multirow{2}{*}{\citet{kidger20}} & $C([0,1]^{d_x}, \mathbb R^{d_y})$ & conti. nonpoly.$^\dagger$ & $w_{\min} \le d_x + d_y + 1$ \\
                                  & $C([0,1]^{d_x}, \mathbb R^{d_y})$ & nonaffine poly. &  $w_{\min} \le d_x + d_y + 2$ \\
 \hline
 \multirow{2}{*}{\citet{park21}} & $L^p(\mathbb R^{d_x}, \mathbb R^{d_y})$ & $\relu$ &  $w_{\min} = \max\{d_x + 1, d_y\}$ \\ 
                                & $L^p([0,1]^{d_x},\mathbb R^{d_y})$ & conti. nonpoly.$^\dagger$ & $w_{\min} \le \max\{ d_x + 2, d_y+1 \}$\\
 \hline
{\citet{kim24}} & $L^p([0,1]^{d_x}, \mathbb R^{d_y})$ &$\relul^{\ddagger*}$ & $w_{\min} = \max\{d_x, d_y, 2\}$ \\
 \hline
 \hline
 \rowcolor{gray!30} {\bf Ours (\cref{thm:main})} & $L^p([0,1]^{d_x}, \mathbb R^{d_y})$ & Squashable$^{\mathsection*}$& $w_{\min} = \max\{d_x, d_y, 2\}$\\
 \hline
\end{tabular}
\end{center}
{\footnotesize
$\|$ requires that $\sigma$ is uniformly approximated by a sequence of one-to-one functions.\\
$\dagger$ requires that $\sigma$ is continuously differentiable at some point $z$, with $\sigma'(z) \neq 0$.\\
$\ddagger$ denotes $\relu$, leaky-$\relu$, $\elu$, $\softplus$, $\celu$, $\selu$, $\gelu$, $\silu$, and $\mish$. \\%$d_x+d_y\ge3$ is required for non-monotone ones. \\
$\mathsection$ includes all analytic functions and a class of piecewise functions such as leaky-$\relu$ (see \cref{sec:squashability,sec:squashability-ex}).\\
$*$ $d_x+d_y\ge3$ is required for non-monotone activation functions.
}
\end{table*}

The minimum width for universal approximation has been studied for two function spaces $C(\mcX,\mcY)$ and $L^p(\mcX,\mcY)$:  $C(\mcX,\mcY)$ denotes the space of continuous functions from $\mcX$ to $\mcY$ endowed with the supremum norm $\sup_{x\in\mcX}\|f(x)\|_\infty$ and $L^p(\mcX,\mcY)$ denotes the space of $L^p$ functions from $\mcX$ to $\mcY$ endowed with the $L^p$-norm $\smash{\|f\|_{L^p}\defeq\left(\int_\mcX \|f\|_p^p d\mu_{d_x}\right)^{1/p}}$ for $p\ge1$.
Recent studies on the minimum width (say $w_{\min}$) was initiated by \citet{Lu17}.
They show that $d_x+1\le w_{\min}\le d_x+4$ for universally approximating $L^1(\bbR^{d_x},\bbR)$ using $\relu$ networks.
\citet{hanin17} consider universally approximating $C([0,1]^{d_x},\bbR^{d_y})$ using $\relu$ networks and prove $d_x+1\le w_{\min}\le d_x+d_y$.
\citet{johnson19} proves the lower bound $w_{\min}\ge d_x+1$ for an activation function that can be uniformly approximated by a sequence of one-to-one functions.
\citet{kidger20} %
show that for $C([0,1]^{d_x},\bbR^{d_y})$, $w_{\min}\le d_x+d_y+1$ if an activation function is continuous, non-polynomial, and continuously differentiable at some point with non-zero derivative. For non-affine polynomial activation functions, they also show $w_{\min}\le d_x+d_y+2$.
However, the upper bounds in these results are at least $d_x+d_y$, which has a large gap compared to the lower bound $d_x+1$.
Such limitation arises from their universal approximator constructions that use $d_x$ neurons to preserve the $d_x$-dimensional input and $d_y+\alpha$ neurons to compute the $d_y$-dimensional output.

The exact minimum width was first characterized by \citet{park21}. By introducing a new universal approximator construction scheme that does not preserve both the $d_x$-dimensional input and $d_y$-dimensional output at once,
they show $w_{\min}=\max\{d_x+1,d_y\}$ to universally approximate $L^p(\bbR^{d_x},\bbR^{d_y})$ if an activation function is $\relu$. 
For $L^p([0,1]^{d_x},\bbR^{d_y})$, they also show $w_{\min}\le\max\{d_x+2,d_y+1\}$ for a class of continuous non-polynomial activation functions. 
Such a scheme has also been applied to other activation functions.
For leaky-$\relu$ networks, \citet{cai23} show that $w_{\min}=\max\{d_x,d_y,2\}$ for $L^p([0,1]^{d_x},\bbR^{d_y})$.
For variants of $\relu$ (see the footnote $\ddagger$ in \cref{table:summary}), \citet{kim24} show $w_{\min}=\max\{d_x,d_y\}$ unless both $d_x$ and $d_y$ are one. They also show $w_{\min}=2$ for $d_x=d_y=1$ if an activation function is monotone.
However, the exact minimum width is only known for $\relu$ and its variants and is unknown for general activation functions.

\subsection{Summary of contributions}

In this work, we study the minimum width enabling universal approximation of $L^p([0,1]^{d_x},\bbR^{d_y})$ using general activation functions. 
Specifically, we consider activation functions $\sigma$ such that an alternative composition of $\sigma$ and affine transformations can approximate the identity function and binary step function $\step(x)$;\footnote{$\step(x)=1$ if $x\ge0$ and $\step(x)=0$ otherwise.} we call such functions \emph{squashable} (see \cref{def:squashable}).
Using the squashability of an activation function $\sigma$, we show that the minimum width of $\sigma$ networks to universally approximate $L^p([0,1]^{d_x},\bbR^{d_y})$ is exactly $\max\{d_x,d_y\}$ unless $d_x=d_y=1$ (\cref{thm:main}). We also show $w_{\min}=2$ when $d_x=d_y=1$ if the squashable function $\sigma$ is monotone.

Our result can be used to characterize the minimum width for a general class of practical activation functions, by showing their squashability. 
For example, we show that \emph{any non-affine analytic function} (e.g., non-affine polynomial, $\sigmoid$, $\tanh$, $\sin$, $\exp$, etc.) is squashable (\cref{lem:analytic}). Furthermore, we also show that a wide class of piecewise continuously differentiable functions including leaky-$\relu$ and $\hswish$ are also squashable (\cref{lem:pieceanalytic}).
Hence, our result significantly extends the prior exact minimum width results %
for $\relu$ and its variants.

Even if an activation is not analytic or piecewise continuously differentiable, it can be squashable, i.e., our minimum width result can be applicable. To check the squashability of general functions, we also provide a sufficient condition for the squashability: $\sigma$ is squashable if and only if there exists an alternative composition $f$ of $\sigma$ and affine transformations such that $f$ is strictly increasing and has a locally sigmoidal shape on some proper interval (\cref{lem:squashable-passingline}).

\subsection{Organization}
We first introduce notations and the problem setup in \cref{sec:notation}. We then formally define the squashability of activation functions, describe our main result on minimum width for universal approximation, and provide sufficient conditions for the squashability in \cref{sec:main}. 
We prove our main result in \cref{sec:pfthm:main} and conclude the paper in \cref{sec:conclusion}. Proofs of technical lemmas are deferred to Appendix.

\section{Problem setup and notations}\label{sec:notation}

In this section, we introduce notations and our problem setup. %
For $n\in\bbN$, we use $[n]$ to denote $\{1,\dots,n\}$.
For $\mcS,\mcT \subset\bbR^d$, we use $\diam(\mcS) \defeq \sup_{x,y\in\mcS}\|x-y\|_\infty$ and $\dist{\mcS}{\mcT} \defeq \inf_{x\in\mcS,y\in\mcT}\|x-y\|_\infty$. If $\mcS$ is a singleton set, (i.e., $\mcS=\{s\}$), we use $\dist{s}{\mcT}$ to denote $\dist{\{s\}}{\mcT}$.
For $y \in \bbR^d$ and $\mcS \subset \bbR^d$, $\mcB_{r}(y)\defeq\{x \in \bbR^d:\dist{x}{y}\le r\}$ and $\mcB_{r}(\mcS)\defeq\{x \in \bbR^d:\dist{x}{\mcS}\le r\}$. %
For a function $f:\bbR^d\to\bbR^{d'}$, $f(x)_i$ denotes the $i$-th coordinate of $f(x)$. For $n\in\bbN$, we use $f^{n}$ to denote the $n$ times composition of $f$.
We use $\iota:\bbR\to\bbR$ to denote the identity function (i.e., $\iota(x)=x$) and $\step$ to denote the binary threshold function (i.e., $\step(x)=0$ if $x<0$ and $\step(x)=1$ otherwise).
We note that all intervals in this paper are proper, i.e., they are neither empty (e.g. $(a,a)=\emptyset$) nor degenerate (e.g. $[a,a]=\{a\}$).

\subsection{Fully-connected networks}

Throughout this paper, we consider fully-connected neural networks. %
Formally, given a set of activation functions $\Sigma$, we define an $L$-layer neural network $f$ with input dimension $d_0=d_x$, output dimension $d_L=d_y$, and hidden layer dimensions $d_1, \cdots, d_{L-1}$ as 
$$f \defeq t_L\circ \tilde\sigma_{L-1}\circ t_{L-1}\circ\cdots\circ\tilde\sigma_1 \circ t_1$$
where $t_\ell:\bbR^{d_{\ell-1}}\to \bbR^{d_\ell}$ is an affine transformation and $\tilde\sigma_\ell(x_1,\dots,x_{d_\ell}) = (\sigma_{\ell,1}(x_1), \cdots, \sigma_{\ell,d_\ell}(x_{d_\ell}))$ for some 
$\sigma_{\ell,1},\dots,\sigma_{\ell,d_\ell}\in\Sigma$
for all $\ell\in[L]$.
We denote a neural network using a single activation function $\sigma$ (i.e., $\Sigma=\{\sigma\}$) by a ``$\sigma$ network'' and a neural network using a two activation functions $\sigma_1,\sigma_2$ (i.e., $\Sigma=\{\sigma_1,\sigma_2\}$) by a ``$(\sigma_1,\sigma_2)$ network''. 
Here, the \emph{width} $w$ of $f$ is defined as the maximum over the hidden dimensions $d_1, \cdots, d_{L-1}$.

We say ``$\sigma$ networks of width $w$ are dense in $L^p(\mcX,\mcY)$'' if for any $f^*\in L^p(\mcX,\mcY)$ and $\varepsilon>0$, there exists a $\sigma$ network of width $w$ such that $\|f^*-f\|_{L^p} \le \varepsilon$.
Given an activation function $\sigma$ and $d_x,d_y\in\bbN$, we use $w_{\sigma,d_x,d_y}$ to denote the minimum $w\in\bbN$ satisfying the following: $\sigma$ networks of width $w$ are dense in $L^p([0,1]^{d_x},\bbR^{d_y})$ but $\sigma$ networks of width $w-1$ are not dense.
We often drop $d_x,d_y$ and use $w_\sigma$ if $d_x,d_y$ are clear from the context.

\section{Main results}\label{sec:main}

\subsection{Squashable activation functions}\label{sec:squashability}

To formally state our main result, we first introduce a class of activation functions that we mainly focus on. To this end, we first introduce the following conditions for an activation function $\sigma$.

\begin{condition}\label{cond:id}
There exists $z\in\bbR$ such that $\sigma$ is continuously differentiable at $z$ and $\sigma'(z)\ne0$.
\end{condition}

\begin{condition}\label{cond:step}
$\sigma$ is continuous and for any compact set $\mcK\subset\bbR$ and for any $\varepsilon, \zeta>0$, there exists a $\sigma$-network $\rho_{\varepsilon, \zeta}:\bbR\to\bbR$ of width $1$ such that
\vspace{-0.1in}
\begin{itemize}[leftmargin=15pt]
    \item $\max_{x\in\mcK\setminus(-\zeta,\zeta)}|\rho_{\varepsilon, \zeta}(x)-\step(x)|\le\varepsilon$,
    \item 
    $\rho_{\varepsilon, \zeta}$ is strictly increasing on $\mcK$, and
    \item $\rho_{\varepsilon, \zeta}(\mcK)\subset[0,1]$.
\end{itemize}
\end{condition}

\cref{cond:id} is that an activation function $\sigma$  is has a continuously differentiable point with a nonzero derivative. This property enables us to approximate the identity function on a compact domain by composing $\sigma$ with affine transformations as stated in the following lemma.
\begin{lemma}[Lemma 4.1 in \cite{kidger20}]\label{lem:kidger}
For any $\varepsilon>0$, $\sigma:\bbR\to\bbR$ satisfying \cref{cond:id}, and compact set $\mcK \subset \bbR$, there exist affine transformations $h_1, h_2:\mcK\to\bbR$ such that 
$$\sup_{x\in\mcK}\|h_2\circ \sigma \circ h_1(x) - x \| \le \varepsilon.$$
\end{lemma}
\cref{cond:step} assumes the continuity of $\sigma$ and the existence of a $\sigma$ network of width $1$ (i.e., an alternative composition of affine transformations and $\sigma$) that can approximate the binary threshold function (i.e., $\step$) on any compact set, except for a small neighborhood of zero (i.e., $(-\zeta,\zeta)$).
One important property in \cref{cond:step} is that $\rho_{\varepsilon,\zeta}$ should be strictly increasing on $\mcK$. This allows 
$\rho_{\varepsilon,\zeta}$ to preserve the information of inputs in $\mcK$ since it is bijective on $\mcK$.

Using these conditions, we now define the \emph{squashability} of an activation function.
\begin{definition}\label{def:squashable}
A function $\sigma:\bbR\to\bbR$ is ``squashable'' if $\sigma$ satisfies \cref{cond:id,cond:step}.
\end{definition}
One can observe that width-$1$ networks using a squashable activation function can approximate the identity function on any compact domain and the $\step$ function on any compact domain except for a small open neighborhood.

A class of squashable activation functions covers a wide range of practical functions.
\cref{cond:id} can be easily satisfied: e.g., any piecewise differentiable function with a non-constant piece satisfies \cref{cond:id}.
Furthermore, we prove that any analytic activation function (e.g., $\sigmoid$, $\exp$, $\sin$.) and a class of piecewise continuously differentiable functions (e.g., leaky-$\relu$ and $\hswish$) %
satisfy \cref{cond:step}. We formally state these results and easily verifiable conditions for \cref{cond:step} in \cref{sec:squashability-ex}.

\subsection{Minimum width with squashable functions}\label{sec:main-thm}

We are now ready to introduce our main theorem on the minimum width for universal approximation. %
\begin{theorem}\label{thm:main}
Let $\sigma$ be a squashable function.
Then, $w_{\sigma}=\max\{d_x,d_y\}$ if $d_x\ge2$ or $d_y\ge2$ and $w_{\sigma}\in\{1,2\}$ if $d_x=d_y=1$.
Furthermore, if $\sigma$ is monotone, then $w_{\sigma}=2$ if $d_x=d_y=1$.
\end{theorem}

\cref{thm:main} characterizes the exact minimum width enabling universal approximation for squashable activation functions: $w_{\sigma,d_x,d_y}=\max\{d_x,d_y\}$ unless the input/output dimensions are both one.
Furthermore, it fully characterizes $w_{\sigma,d_x,d_y}$ for all $d_x,d_y$ if an activation function is squashable and monotone.
The proof of \cref{thm:main} is in \cref{sec:pfthm:main}.

To the best of our knowledge, the exact minimum width enabling universal approximation has been discovered only for a few $\relul$ activation functions such as $\relu, \text{leaky-}\relu,\softplus, \gelu$ \citep{park21,cai23,kim24}.
Furthermore, the best known upper bound for a general class of activation functions was $w_\sigma\le\max\{d_x+2,d_y+1\}$ when $\sigma$ is continuous non-polynomial and continuously differentiable at some point with non-zero derivative \citep{park21}.
Our result extends prior exact minimum width results to a general class of activation functions (i.e., squashable) including all analytic functions (e.g., $\sigmoid$, $\tanh$, $\sin$, $\exp$, polynomial) and a class of piecewise continuously differentiable functions (e.g., $\hswish$). See \cref{lem:analytic,lem:pieceanalytic} in \cref{sec:squashability-ex} for more details on squashable functions.

\subsection{Easily verifiable conditions for \cref{cond:step}}\label{sec:squashability-ex}

In \cref{thm:main}, we have observed that $w_{\sigma,d_x,d_y}$ can be characterized if $\sigma$ is squashable.
However, checking whether a given activation function is squashable, especially whether it satisfies \cref{cond:step}, can be non-trivial. 
In this section, we provide easily verifiable conditions for \cref{cond:step} based on the following lemma.

\begin{figure*}
    \centering    
        \begin{subfigure}[b]{0.28\linewidth}
         \centering
        \includegraphics[width=1.0\linewidth]{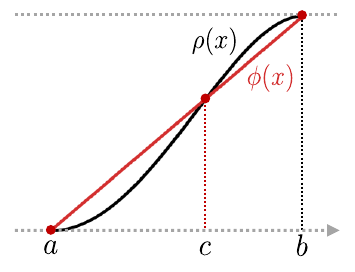}
         \caption{}
         \label{fig:fixedpoint1}
         \end{subfigure}
         \qquad
         \begin{subfigure}[b]{0.28\linewidth}
         \centering
         \includegraphics[width=1.0\linewidth]{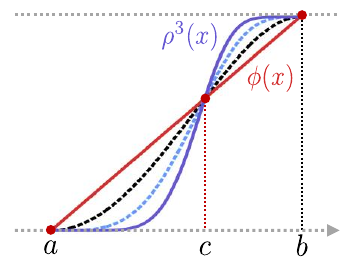}  
         \caption{}
         \label{fig:fixedpoint2}
         \end{subfigure}
         \qquad
          \begin{subfigure}[b]{0.28\linewidth}
        \centering
        \includegraphics[width=1.0\linewidth]{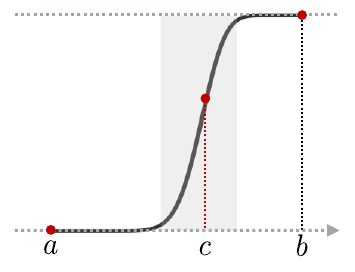}
         \caption{}
         \label{fig:fixedpoint3}
         \end{subfigure}

    \caption{Illustration of construction of squashable function using a $\sigma$ network $\rho$ of width $1$ that has a sigmoidal shape when $\phi(x)=x$. %
    The intersections of $\rho(x)$ and $\phi(x)$ %
    serve as \emph{fixed points}. Thus, $\sigma$ can achieve the squashability by iteratively composing $\rho$: $\rho^n(x)\to a$ for $x\in(a,c)$ and $\rho^n(x)\to b$ for $x\in(c,b)$ as $n\to\infty$ while $\rho^n$ is strictly monotone. %
    }

    \label{fig:fixedpoint}
\end{figure*}

\begin{lemma}\label{lem:squashable-passingline}
A continuous function $\sigma:\bbR\to\bbR$ satisfies \cref{cond:step} if there exist a $\sigma$ network $\rho$ of width $1$ and $a,b\in\bbR$ with $a<b$ satisfying the following:
\vspace{-0.1in}
\begin{itemize}[leftmargin = 15pt]
\item $\rho$ is strictly increasing on $[a,b]$ and
\item there exists $c\in(a,b)$ such that
\begin{align*}
\rho(x)<\phi(x) \ \forall x\in(a,c),\quad \rho(x)>\phi(x)\ \forall x\in(c,b)
\end{align*}
where $\phi(x)$ is a line passing $(a,\rho(a))$ and $(b,\rho(b))$.
\end{itemize}
\end{lemma}

\cref{lem:squashable-passingline} provides a sufficient condition for \cref{cond:step}: if we can make a $\sigma$ network of width $1$ that has a ``sigmoidal shape'' on some compact domain (e.g., see \cref{fig:fixedpoint1}), then $\sigma$ satisfies \cref{cond:step}. 
We can easily approximate the $\step$ function using a function with the sigmoidal shape by composing the function and some affine transformations (see \cref{fig:fixedpoint2,fig:fixedpoint3}).
For a more formal argument, see the proof of \cref{lem:squashable-passingline} in \cref{sec:pflem:passingline}.

Such a sigmoidal shape (or its symmetric variants) exists in various smooth activation functions such as $\gelu$, $\sigmoid$, $\tanh$, and $\sin$. In addition, for any non-affine  analytic function $\sigma$, we can always make a $\sigma$ network of width $1$ that has the sigmoidal shape. Since all non-constant analytic functions are continuously differentiable and have a non-zero derivative at some point, all non-affine analytic functions satisfy \cref{cond:id,cond:step}, i.e., they are squashable.
The proof of \cref{lem:analytic} is presented in \cref{sec:pflem:analytic}.
\begin{lemma}\label{lem:analytic}
All non-affine analytic functions from $\bbR$ to $\bbR$ satisfy \cref{cond:step}. 
\end{lemma}
In addition, a class of piecewise functions also satisfies the condition in \cref{lem:squashable-passingline}.
We defer the proof of \cref{lem:pieceanalytic} to \cref{sec:pflem:pieceanalytic}.
\begin{lemma}\label{lem:pieceanalytic}
A continuous function $\sigma:\bbR\to\bbR$ satisfies \cref{cond:step} if there exist $c\in\bbR$ and $\delta>0$ such that 
\begin{itemize}
\item $\sigma$ is continuously differentiable on $(c-\delta, c+\delta)\setminus\{c\}$,
\item $v^+=\lim_{x\to c^-}\sigma'(x)$ and $v^-=\lim_{x\to c^+}\sigma'(x)$ exist, $v^+\ne v^-$, and $v^+v^->0$.
\end{itemize}
\end{lemma}
\cref{lem:pieceanalytic} states that if an activation function $\sigma$ contains a point such that the left limit of the derivative and the right limit of derivative at that point are different but have the same sign, then $\sigma$ satisfies \cref{cond:step}.
We note that piecewise functions such as leaky-$\relu$ and $\hswish$ satisfy the condition in \cref{lem:pieceanalytic}; for those functions, one can choose the point $c$ in \cref{lem:pieceanalytic} as some break point between consecutive pieces.

While we provide easily verifiable sufficient conditions (\cref{lem:analytic,lem:pieceanalytic}) for \cref{cond:step}, we note that \cref{thm:main} covers any activation function satisfying \cref{cond:id,cond:step}, even if that activation function does not satisfy conditions in \cref{lem:analytic,lem:pieceanalytic}.
We also present additional sufficient conditions for \cref{cond:step} in \cref{sec:additionalcondition}.

\section{Proof of \cref{thm:main}}\label{sec:pfthm:main}

We now present the proof of \cref{thm:main}.
\cref{thm:main} is a direct corollary of the following lemmas.

\begin{lemma}\label{lem:ub}
Let $\sigma$ be a squashable function, $\varepsilon>0$, $f^*\in C([0,1]^{d_x},[0,1]^{d_y})$, and $p\ge 1$.
Then, there exists a $\sigma$ network $f:[0,1]^{d_x}\to\bbR^{d_y}$ of width $\max\{d_x, d_y, 2\}$ such that $$\|f-f^*\|_{L^p}\le \varepsilon.$$
\end{lemma}

\begin{lemma}[Lemmas 21 and 22 in \cite{kim24}]\label{lem:lb}
For any $\sigma:\bbR\to\bbR$ and $d_x,d_y\in\bbN$, $w_{\sigma}\ge\max\{d_x,d_y\}$.
Furthermore, if $\sigma$ is monotone, then $w_{\sigma}\ge2$.
\end{lemma}

\cref{lem:ub} implies that for any squashable activation function $\sigma$, $w_\sigma\le\max\{d_x,d_y,2\}$. This is because (1) continuous functions on $[0,1]^{d_x}$ are dense in $L^p([0,1]^{d_x},\bbR^{d_y})$ \citep{rudin} and (2) $g([0,1]^{d_x})$ is compact for all $g\in C([0,1]^{d_x},\bbR^{d_y})$, i.e., we can scale the range of $g$ to be in $[0,1]^{d_y}$.
Hence, combining \cref{lem:ub,lem:lb} results in \cref{thm:main}. In the rest of this section, we prove \cref{lem:ub}.

\subsection{Proof of \cref{lem:ub}}
To illustrate our main idea for proving \cref{lem:ub}, we first define a \emph{$\delta$-filling curve}.

\begin{definition}\label{def:filling-curve}
Let $d \in \bbN$ and $\delta>0$. We say a continuous function $f:\bbR\to\bbR^d$ is a ``$\delta$-filling curve'' of $\mcD \subset \bbR^d$ if
$$\sup_{y \in \mcD} \dist{y}{f([0,1])} \le \delta.$$
\end{definition}

A $\delta$-filling curve of $\mcD\subset\bbR^d$ can be considered as a weaker version of a space-filling curve of $\mcD$ \citep{sfc}. While the range of the space-filling curve contains $\mcD$ but the $\delta$-filling curve covers $\mcD$ within $\delta$ distance.

Suppose that we can implement a $\delta$-filling curve $h$ of $[0,1]^{d_y}$ using a $\sigma$ network for some small $\delta>0$, i.e., for each $y\in[0,1]^{d_y}$, there is $z\in[0,1]$ such that $h(z)\approx y$.
Hence, if we can design a $\sigma$ network $g$ that maps each $x\in[0,1]^{d_x}$ to some $z_x$ such that $h(z_x)\approx f^*(x)$, then the $\sigma$ network $h\circ g$ approximates $f^*$.
Here, $g$ and $h$ can be considered as an \emph{encoder} and $\emph{decoder}$: $g$ encodes a $d_x$-dimensional vector $x$ to a scalar value $z_x$ that contains the information of $f^*(x)$ and $h$ decodes $z_x$ to a $d_y$-dimensional vector $h(z_x)$ that approximates $f^*(x)$.

We explicitly construct networks that approximate the encoder and decoder. To this end, we introduce the following lemma where the proof is deferred to \cref{sec:pflem:simplification}.
\begin{lemma}\label{lem:simplification}
Let $\sigma$ be a squashable function, $d,w\in\bbN$, and $\mcK\subset\bbR^d$ be a compact set. Then, for any $\varepsilon>0$ and $(\sigma,\iota)$ network $f$ of width $w$, there exists a $\sigma$ network $g$ of width $w$ such that $$\sup_{x\in\mcK}\|f(x)-g(x)\|_\infty<\varepsilon.$$
\end{lemma}
Here, $\iota:\bbR\to\bbR$ denotes the identity function (see \cref{sec:notation}).
\cref{lem:simplification} implies that constructing a $(\sigma,\iota)$ network of width $\max\{d_x,d_y,2\}$ that approximates $f^*$ is sufficient to prove \cref{lem:ub}. Hence, we focus on approximating the encoder and decoder using $(\sigma,\iota)$ networks.

We first show that the decoder can be implemented using a $(\sigma,\iota)$ network of width $d_y$.
The proof of \cref{lem:decoder} is in \cref{sec:pflem:decoder}.

\begin{lemma}\label{lem:decoder}
Let $\sigma$ be a squashable function and $\delta>0$. Then, there exists a $(\sigma,\iota)$ network $f_\mathrm{dec}:[0,1]\to[0,1]^{d_y}$ of width $d_y$ that is a $\delta$-filling curve of $[0,1]^{d_y}$.
\end{lemma} 
\cref{lem:decoder} states that for any $\delta>0$, we can always implement a $\delta$-filling curve of $[0,1]^{d_y}$ using 
a $(\sigma,\iota)$ network $f_{\mathrm{dec}}$ of width $d_y$.
Further, the implemented network satisfies
$f_{\mathrm{dec}}([0,1])\subset[0,1]^{d_y}$ regardless of $\delta$.

We also show that the encoder can be approximated by a $(\sigma,\iota)$ network of width $\max\{d_x,2\}$. The proof of \cref{lem:encoder} is in \cref{sec:pflem:encoder-sketch}.
\begin{lemma}\label{lem:encoder}
Let $\sigma$ be a squashable function, $N\in\bbN$ and $\gamma\in(0,0.5)$.
For each $\nu\in[N]^{d_x}$, let $\mcT_{\nu}=\prod_{i=1}^{d_x}[\frac{\nu_i-1+\gamma}{N},\frac{\nu_i-\gamma}{N}]$ and $c_\nu\in[0,1]$. Then, there exists a $(\sigma, \iota)$ network $f_{\mathrm{enc}}:[0,1]^{d_x} \to [0,1]$ of width $\max\{d_x,2\}$ such that for each $\nu \in [N]^{d_x}$,
$$f_{\mathrm{enc}}(\mcT_\nu) 
\subset \mcB_{\gamma}(c_\nu).$$
\end{lemma}
\begin{figure*}
    \centering    \includegraphics[width=\linewidth]{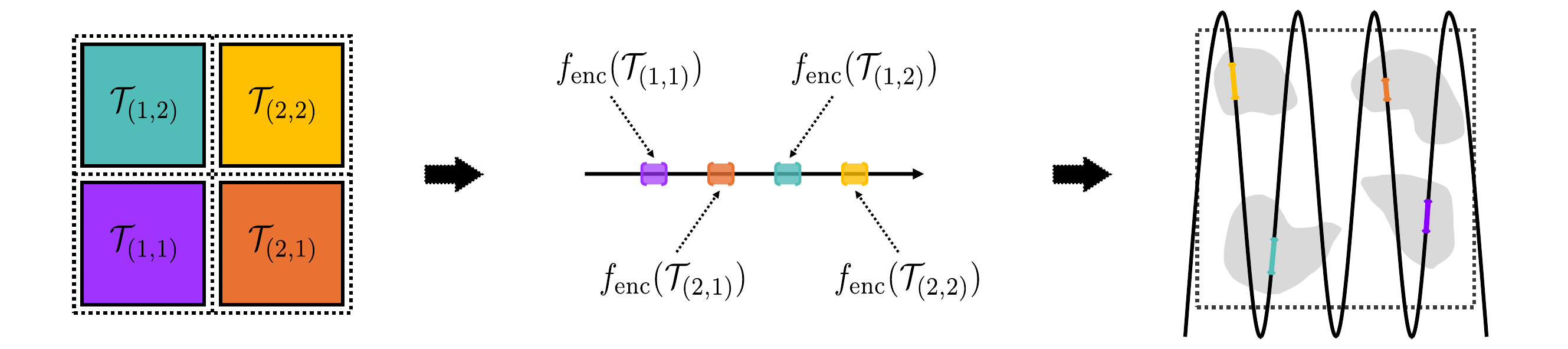}
    \caption{
    Illustration of $f_\text{dec}\circ f_\text{enc}$ when $d_x=2$, $d_y=2$, and $N=2$.
    $f_\mathrm{enc}$ first encodes each $\mcT_\nu$ to a bounded interval $f_\mathrm{enc}(\mcT_\nu)$. 
    Then, $f_\mathrm{dec}$ implements \emph{$\delta$-filling curve} of $[0,1]^2$,
    represented by the black curve, to decode each  $f_\mathrm{enc}(\mcT_\nu)$ (colored) that approximates $f^*(\mcT_\nu)$ (represented by the light gray area).
    }
    \label{fig:coding}
\end{figure*}

The collection of $\mcT_\nu$ in \cref{lem:encoder} 
can be regarded as an approximate partition of $[0,1]^{d_x}$:
its elements are disjoint and it covers almost all parts of the domain with a small enough $\gamma>0$.
By choosing a large enough $N$, the diameter of $f^*(\mcT_\nu)$ can be arbitrarily small, i.e., $f^*(x)\approx f^*(x')$ for all $x,x'\in\mcT_\nu$. 
Under this setup, choose $c_\nu$ for each $\nu$ so that $f_\text{dec}(c_\nu)\approx f^*(\mcT_\nu)$. %
Then,
$f_{\mathrm{enc}}$ in \cref{lem:encoder} maps each element $\mcT_\nu$ in the approximate partition to some small ball centered at $c_\nu$, with diameter $\gamma$.
Since $f_\text{dec}$ is continuous, this implies that for each $x\in\mcT_\nu$, $f_\text{dec}\circ f_\text{enc}(x)\approx f^*(x)$ with small enough $\delta$ for $f_\text{dec}$ and small enough $\gamma$, large enough $N$ for $f_\text{enc}$.
See \cref{fig:coding} for the illustration.
Here, we note that $f_\text{dec}\circ f_\text{enc}$ is a $(\sigma,\iota)$ network of width $\max\{d_x,d_y,2\}$.

For $x\notin\bigcup_\nu\mcT_\nu$, we have $f_\text{dec}\circ f_\text{enc}(x)\in[0,1]^{d_y}$ (i.e., bounded) by \cref{lem:decoder,lem:encoder}.
Since $\mu_{d_x}([0,1]^{d_x}\setminus(\bigcup_\nu\mcT_\nu))\to0$ as $\gamma\to0$, one can observe that for any $\varepsilon>0$, there exist small enough $\gamma,\delta$ and large enough $N$ such that $\|f_\text{dec}\circ f_\text{enc}-f^*\|_{L^p}\le\varepsilon$. Namely, a $(\sigma,\iota)$ network $f=f_\text{dec}\circ f_\text{enc}$ has width $\max\{d_x,d_y,2\}$ and completes the proof.
Given $\varepsilon>0$, our explicit choices of $\delta,\gamma,N$ and the detailed derivation of $\|f_\text{dec}\circ f_\text{enc}-f^*\|_{L^p}\le\varepsilon$ %
can be found in \cref{sec:error-analysis}.

\begin{figure*}
     \centering
         \begin{subfigure}[b]{0.45\linewidth}
             \centering
             \includegraphics[width=\linewidth]{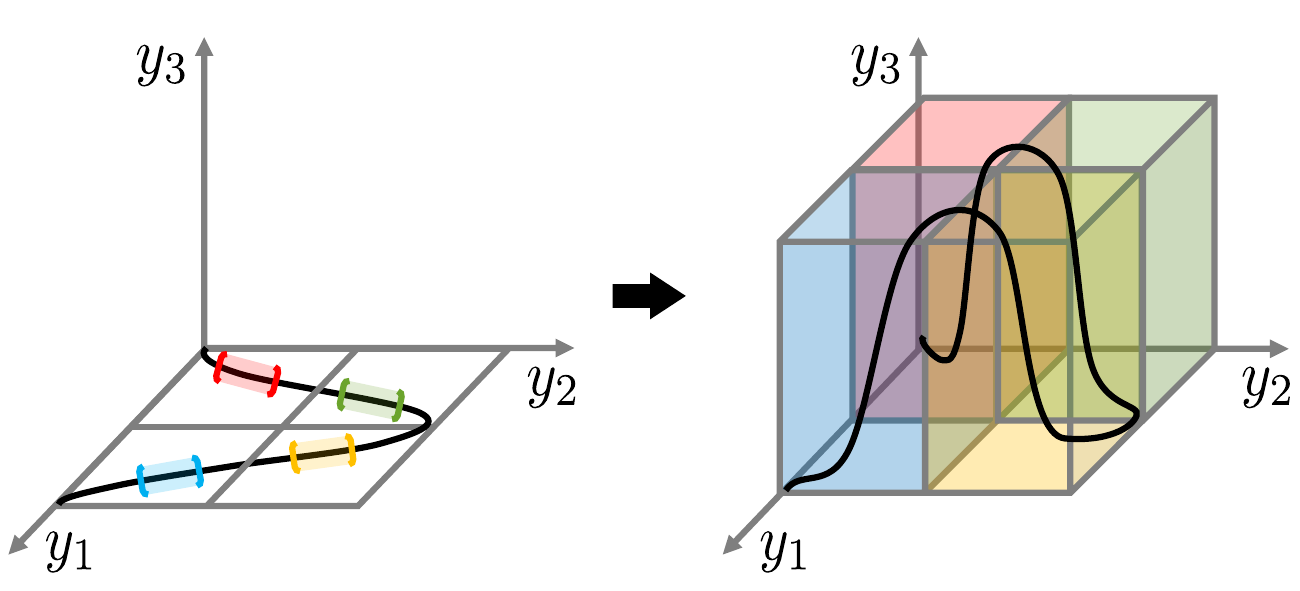}
             \caption{}
             \label{fig:2to3}
         \end{subfigure}
         \qquad
         \begin{subfigure}[b]{0.45\linewidth}
         \centering
         \includegraphics[width=\linewidth]{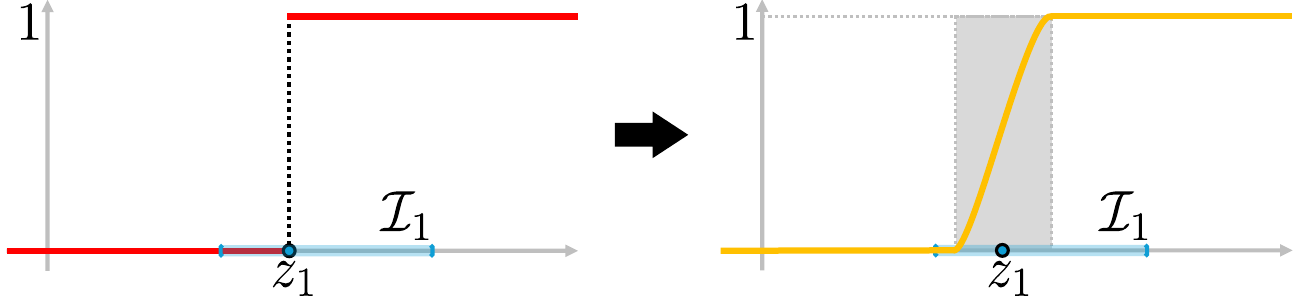}
         \caption{}
         \label{fig:step1}
         \vfill
         \includegraphics[width=\linewidth]{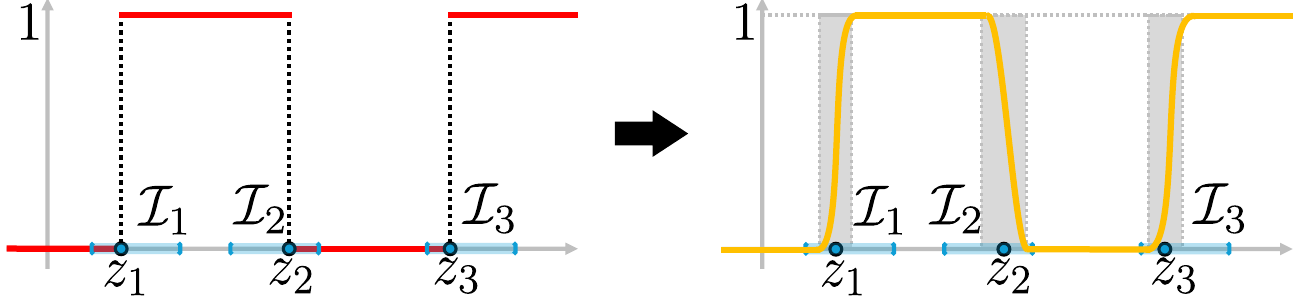}
         \caption{}
         \label{fig:step2}
         \end{subfigure}
    \caption{(a) Illustration of a $(1/N)$-filling curve $\tilde f$ of $[0,1]^3$. $\tilde f$ maps each open interval $\mcI_\nu$, represented by the colored brackets (left), to be intersected with the corresponding cube of the same color (right).
    (b) and (c) illustrates our network $\rho$ satisfying the properties of $\phi$ when $N=1$ and $N=3$, respectively.
            }
    \label{fig:decoder}
\end{figure*}

\subsection{Proof of \cref{lem:decoder}}\label{sec:pflem:decoder}

In this section, we prove \cref{lem:decoder} by showing the following: for each $N,d\in\bbN$, there exists a $(\sigma,\iota)$ network of width $d_y$ that is a $(1/N)$-filling curve of $[0,1]^d$. %
In particular, we inductively construct a $(1/N)$-filling curve of $[0,1]^d$ from $d=1$. %
Here, the base case $d=1$ is trivial: a $(\sigma,\iota)$ network $f(x)=\iota(x)$ is a $(1/N)$-filling curve of $[0,1]$ for all $N\in\bbN$.

We prove the general case ($d\ge 2$) using the inductive step described in the following lemma, whose formal proof is in \cref{sec:pflem:decoder-inductive}. Here, for $N,d\in\bbN$, we use $\mcC_{N,d,\nu}\defeq\prod_{i=1}^d[\frac{\nu_i-1}{N},\frac{\nu_i}{N}]$ and $\nu=(\nu_1,\dots,\nu_d)\in[N]^d$.

\begin{lemma}\label{lem:decoder-inductive}
Let $N,d\in\bbN$ and $\sigma$ be a squashable function.
Suppose that there exist disjoint open intervals $\mcI_{\nu}\subset[0,1]$ for all $\nu\in[N]^d$ and a $(\sigma,\iota)$ network $f:[0,1]\to[0,1]^{d}$ of width $d$ such that for each $x\in[0,1]$ and $\nu\in[N]^d$,
\begin{align*}
f(x)_1=x~~\text{and}~~f(\mcI_\nu)\subset\mcC_{N,d,\nu}.
\end{align*}
Then, there exist disjoint open intervals $\mcJ_{\tilde \nu}\subset[0,1]$ for all $\tilde \nu\in[N]^{d+1}$ and a $(\sigma,\iota)$ network $\tilde f:[0,1]\to[0,1]^{d+1}$ of width $d+1$ such that for each $x\in[0,1]$ and $\tilde \nu\in[N]^{d+1}$,
\begin{align*}
\tilde f(x)_1=x~~\text{and}~~\tilde f(\mcJ_{\tilde \nu})\subset\mcC_{N,d+1,\tilde \nu}.
\end{align*}
\end{lemma}
One can observe that $(\sigma,\iota)$ networks $f$ and $\tilde f$ in \cref{lem:decoder-inductive} are $(1/N)$-filling curves of $[0,1]^{d}$ and $[0,1]^{d+1}$, respectively. Furthermore, our filling curve construction $f(x)=\iota(x)$ for the base case satisfies the assumption in \cref{lem:decoder-inductive} with $\mcI_\nu=(\frac{\nu-1}{N},\frac{\nu}{N})$ for all $N\in\bbN$ and $\nu\in[N]$.
Hence, by \cref{lem:decoder-inductive}, we can conclude that for each $N,d\in\bbN$, there exists a $(\sigma,\iota)$ network that is a $(1/N)$-filling curve of $[0,1]^d$; this proves \cref{lem:decoder}.

We now briefly illustrate our main idea for constructing $\tilde f$ in \cref{lem:decoder-inductive} given $f$. %
Suppose that disjoint open intervals $\mcI_\nu$ for all $\nu\in[N]^d$ and corresponding $(\sigma,\iota)$ network $f$ of width $d$ in \cref{lem:decoder-inductive} are given.
Then, to prove \cref{lem:decoder-inductive}, it suffices to construct a $(\sigma,\iota)$ network $\tilde f$ of width $d+1$ such that for each $i\in[d]$ and $\nu\in[N]^d$,
\begin{align*}
\tilde f(x)_i=f(x)_i~~\text{and}~~[\tfrac1{2N},1-\tfrac1{2N}]\subset \tilde f(\mcI_\nu)_{d+1}.
\end{align*}
This implies that if we can construct a $(\sigma,\iota)$ network $\phi:[0,1]\to\bbR^2$ of width $2$ such that for each $\nu\in[N]^d$,    
\begin{align}
\phi(x)_1=x~\text{and}~[\tfrac1{2N},1-\tfrac1{2N}]\subset \phi(\mcI_\nu)_2,\label{eq:phi}
\end{align}
then we can construct $\tilde f$ in \cref{lem:decoder-inductive} by choosing
\begin{align*}
&\tilde f(x)_1=\phi(f(x)_1)_1,~~\tilde f(x)_{d+1}=\phi(f(x)_1)_2,~\text{and}\\
&\tilde f(x)_i=\iota\circ\cdots\circ\iota(f(x)_i)~\text{for all}~i\in\{2,\dots,d\}.
\end{align*}
See \cref{fig:2to3} for the illustration.
We can construct such $\phi$ using the squashability of $\sigma$.
For example, suppose that $N=1$ and $d=1$ (i.e., there is exactly one $\mcI_\nu$). By \cref{def:squashable}, for any $\varepsilon,\zeta>0$ and compact $\mcK\subset\bbR$ with $[-\zeta,\zeta]\subset\mcK$, there is a width-$1$ $\sigma$ network $\rho$  such that
\begin{align*}
\max_{x\in\mcK\setminus(-\zeta,\zeta)}|\rho(x)-\step(x)|\le\varepsilon.
\end{align*}
Then, by the intermediate value theorem, we have 
$$[\varepsilon,1-\varepsilon]\subset\rho([-\zeta,\zeta]).$$
This implies that by choosing $\tilde \rho(x)=\rho(x-z_\nu)$ for some $z_\nu\in\mcI_\nu$ and $\mcK$ containing $\mcI_\nu$ with small enough $\varepsilon,\zeta>0$, 
it holds that $[\frac1{2N},1-\frac1{2N}]\subset\tilde \rho(\mcI_\nu)$
(see \cref{fig:step1}). In this case, we can choose a width-2 $(\sigma,\iota)$ network $\phi$ satisfying \cref{eq:phi} as $\phi(x)_1=x$ and $\phi(x)_2=\tilde \rho(x)$.

Such a construction also extends to an arbitrary number of $\mcI_\nu$ by composing $\rho$ (i.e., an approximation of $\step$). For example, let $\mcI_1,\mcI_2,\mcI_3\subset[0,1]$ be disjoint open intervals and let $z_i\in\mcI_i$. Then, we have
\begin{align*}
\psi(x) 
&= \step(x-z_1 + (z_1 - z_3) \times \step(x - z_2))\\
&=\begin{cases}
0 & \text{if} ~~ x \le z_1 ~~\text{or}~~ z_2 \le x < z_3\\
1 & \text{otherwise}
\end{cases}.
\end{align*}
Namely, by replacing $\step$ by $\rho$ in $\psi$ with small enough $\varepsilon,\zeta>0$ (and denoting that function by $\tilde \psi$), we have $[\frac1{2N},1-\frac1{2N}] \subset \tilde \psi(\mcI_i)$ by the intermediate value theorem (see \cref{fig:step2}).
We present a more detailed argument for general $N,d$ in the proof of \cref{lem:indicator} in \cref{sec:pflem:indicator}.

\subsection{Proof of \cref{lem:encoder}}\label{sec:pflem:encoder-sketch}

We now prove \cref{lem:encoder}. Our construction of $f_\text{enc}$ consists of two $(\sigma,\iota)$ networks: $f_1:[0,1]^{d_x}\to\bbR$ of width $d_x$ and $f_2:\bbR\to\bbR$ of width $2$. Here, $f_1$ maps each $\mcT_{\nu}$ to a disjoint compact interval $f_1(\mcT_{\nu})$ and $f_2$ is designed to satisfy $f_2(f_1(\mcT_{\nu}))\subset\mcB_\gamma(c_\nu)$ for each $\nu$.
Namely, $f_\text{enc}=f_2\circ f_1$ satisfies $f(\mcT_\nu)\subset\mcB_{\gamma}(c_\nu)$.

{\bf Construction of $f_2$.}
The following lemma shows the existence of $f_2$ such that $f_2(f_1(\mcT_{\nu}))\subset\mcB_\gamma(c_\nu)$ for each $\nu\in[N]^{d_x}$. 
\begin{lemma}\label{lem:piecewise-constant}
Let $\mcK\subset\bbR$ be a compact interval and $\mcI_1,\dots,\mcI_N\subset\mcK$ be disjoint closed subintervals.
Then, for any $\varepsilon>0$, squashable $\sigma$, and $c_1, \dots, c_N \in \bbR$, there exists a $(\sigma, \iota)$ network $f:\mcK \to [0,1]$ of width $2$ such that for each $k \in [N]$,
$$\sup_{x \in \mcI_k} |f(x) - c_k| \le \varepsilon.$$
\end{lemma}
We prove \cref{lem:piecewise-constant} by explicitly constructing a $(\sigma,\iota)$ network that approximates a piecewise constant function which maps each interval $\mcI_k$ to $c_k$. The formal proof of \cref{lem:piecewise-constant} is in \cref{sec:pflem:piecewise-constant}.

{\bf Construction of $f_1$.} In the remainder of this section, we construct a $(\sigma,\iota)$ network $f_1$ of width $d_x$ that maps each $\mcT_\nu$ to a disjoint compact interval $f_1(\mcT_\nu)$. Here, we assume $d_x \ge 2$; if $d_x=1$, we choose $f_1(x)=\iota(x)$. 
To describe our construction we define a \emph{$d$-grid.}
\begin{definition}
A collection of sets $\mcG\subset2^{\bbR^d}$ is a ``$d$-grid'' of size $(n_1,\dots,n_d)\in\bbN^d$ if there exist %
disjoint compact intervals $\mcI_{i,1},\dots,\mcI_{i,n_i}\subset\bbR$ for each $i\in[d]$ such that $$\mcG=\{\mcI_{i,j_1}\times\cdots\times\mcI_{i,j_d}:j_i\in[n_i],~~\forall i\in[d]\}.$$
\end{definition}
One can observe that any finite set of disjoint intervals is a $1$-grid and $\mcT_\nu$ is a $d_x$-grid.
We construct $f_1$ using the following lemma. The proof of \cref{lem:inductive-encoder} is in \cref{sec:pflem:inductive-encoder}. 

\begin{figure*}
    \centering    
         \begin{subfigure}[b]{0.5\linewidth}
         \centering
         \includegraphics[width=1.0\linewidth]{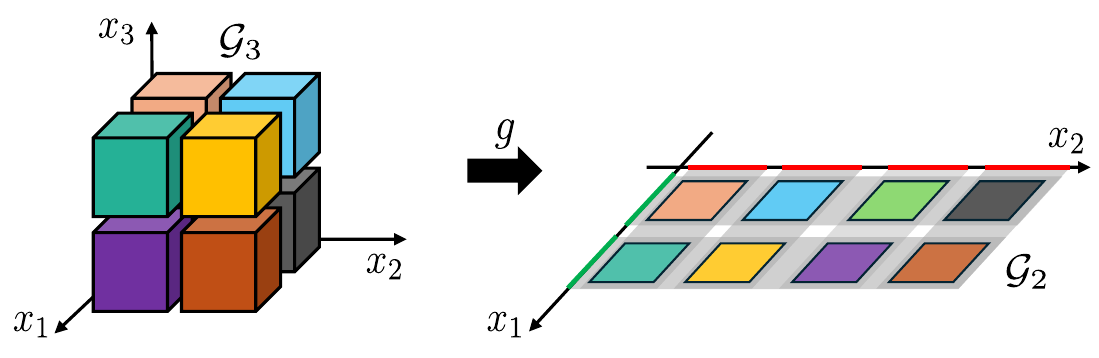} 
         \caption{}
         \label{fig:encoder2}
      \end{subfigure}
      \quad
        \begin{subfigure}[b]{0.4\linewidth}
         \centering
        \includegraphics[width=1.0\linewidth]{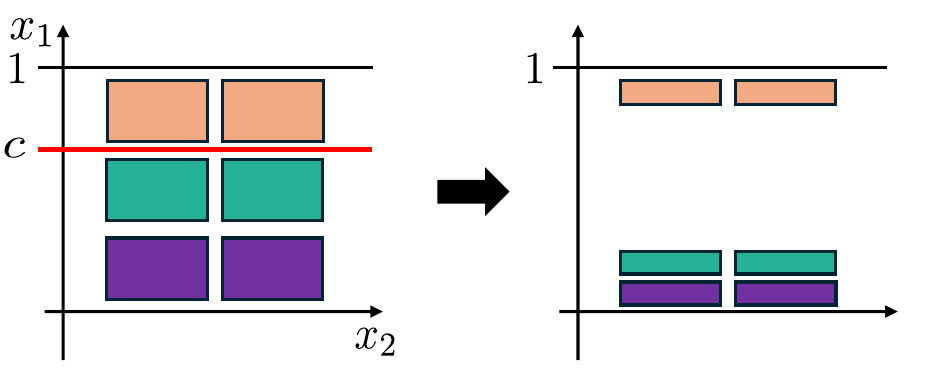}
         \caption{}
         \label{fig:encoder1}

         \end{subfigure}
         \\
         \begin{subfigure}[b]{1.0\linewidth}
         \centering
        \includegraphics[width=1.0\linewidth]{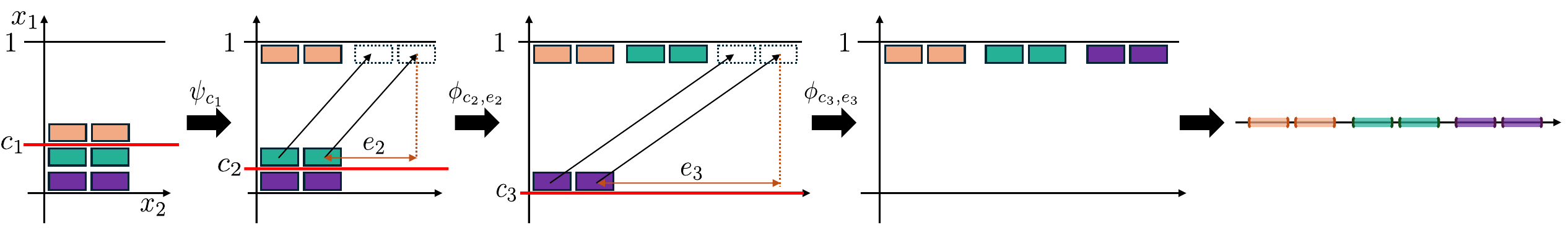}
         \caption{}
         \label{fig:encoder3}
         \end{subfigure}

    \caption{(a) Illustration of a function $g:\bbR^3 \to\bbR^2$ that maps sets in a $3$-grid $\mcG_3$ of size $(2,2,2)$ to distinct sets in $2$-grid $\mcG_2$ of size $(2,4)$. 
    (b) Illustration of $\psi_c:\bbR^2\to\bbR^2$. Here, the first coordinate of $\phi_c(x)$ is approximately $1$ or $0$ depending on whether $x_1$ exceeds $c$ or not while the second coordinate is $x_2$. (c) Illustration of our construction of $f$ when $\mcG$ is a $2$-grid of size $(3,2)$ and $e_2,e_3>0$ are chosen so that all sets in $\mcG$ are disjoint in the second coordinate.}
    \label{fig:encoder}
\end{figure*}

\begin{lemma}\label{lem:inductive-encoder}
Let $\sigma$ be a squashable function and $\mcG$ be a $2$-grid of size $(n_1,n_2)$.
Then, there exist %
a $(\sigma,\iota)$ network $f:\mcK\to\bbR$ of width $2$ such that $\{f(\mcS):\mcS\in\mcG\}$ is an $1$-grid of size $n_1n_2$.
\end{lemma}

\cref{lem:inductive-encoder} implies that there exists a $(\sigma, \iota)$ network $f$ of width $2$ that maps sets in a $2$-grid to sets in an $1$-grid. 
This implies that for any distinct sets $\mcS,\mcS'$ in the $2$-grid, $f(\mcS)\cap f(\mcS')=\emptyset$.
We now construct $f_1$ by using $(\sigma, \iota)$ networks that reduce dimensions one by one while preserving the disjointness of each $\mcT_\nu$. 

We first show that for any $d\ge2$ and $d$-grid $\mcG$ of size $(n_1,\dots,n_d)$, we can construct a $(\sigma,\iota)$ network $g$ of width $d$ that maps sets in the grid to a $(d-1)$-grid of size $(n_1,\dots,n_{d-2},n_{d-1}n_d)$.
Specifically, such $g_d$ can be constructed by using \cref{lem:inductive-encoder}.
Let $\mcG'$ be a $2$-grid defined by considering the last two coordinates of sets in $\mcG$, i.e., $$\mcG'=\big\{\{(x_{d-1},x_d):(x_1,\dots,x_d)\in\mcS\}:\mcS\in\mcG\big\}.$$
Then, $g$ can be constructed as
\begin{align*}
g(x)_i=\begin{cases}
x_i~&\text{if}~i\le d-2\\
\phi(x_{d-1},x_d)_1~&\text{if}~i=d-1\\
\phi(x_{d-1},x_d)_2~&\text{if}~i=d
\end{cases}
\end{align*}
where $\phi$ is a $(\sigma,\iota)$ network of width $2$ in \cref{lem:inductive-encoder} that maps the $2$-grid $\mcG'$ of size $(n_{d-1},n_d)$ to some $1$-grid of size $n_{d-1}n_d$; see \cref{fig:encoder2} for the illustration of $g$ when $d=3$.

Let $\mcG_{d_x}=\{\mcT_\nu:\nu\in[N]^{d_x}\}$ be a $d_x$-grid of size $(N,\dots,N)$.
As in the construction of $g$, we recursively construct $g_{i}$ for $i=d_x,d_x-1,\dots,2$ as a $(\sigma,\iota)$ network of width $i$ that maps an $i$-grid $\mcG_{i}$ of size $(N,\dots,N,N^{d_x-i+1})$ to some $(i-1)$-grid $\mcG_{i-1}$ of size $(N,\dots,N,N^{d_x-i+2})$.
We then construct $f_1$ as $f_1=g_2\circ g_3\circ\cdots\circ g_{d_x}$. One can observe that $f_1$ has width $d_x$ and maps sets in $\mcG_{d_x}$ to distinct sets in some $1$-grid.

{\bf Intuition behind \cref{lem:inductive-encoder}.}
We now briefly describe our main proof idea for \cref{lem:inductive-encoder} where the formal proof is deferred to \cref{sec:pflem:inductive-encoder}. %
Our construction of $f$ is based on the squashability of $\sigma$. 
Observe that by the definition of the squashability (\cref{def:squashable}), for any compact set $\mcK\subset\bbR$, there exists a width-$1$ network $\rho$ that is strictly increasing and approximates $\step$ on $\mcK$ (see \cref{cond:step}).

Consider a width-$2$ network $\psi_c:\bbR^2\to\bbR^2$ defined as $\psi_c(x)=(\rho(x_1-c),x_2)$ for some $c\in\bbR$.
Then, by choosing a proper $c$ and $\mcK$, $\psi$ splits sets in $\mcG$ into two parts depending on whether their first coordinate exceeds $c$ or not. Here, $\psi_c(\mcS)_1$ will be close to one if the first coordinate of $\mcS$ exceeds $c$ and $\psi_c(\mcS)_1$ will be close to zero otherwise.
We note that by the strict monotonicity of $\rho$, the order of the first coordinate of the sets does not change. See \cref{fig:encoder1} for the illustration.

Furthermore, we can also change the second coordinate while splitting the first coordinate. For any $e>0$, by composing $\psi_c$ with some invertible affine transformation $\kappa_e:\bbR^2\to\bbR^2$, we can construct a width-$2$ network $\phi_{c,e}=\kappa^{-1}_e\circ\psi_c\circ\kappa_e$ so that
\begin{align*}
\phi_{c,e}(x)\approx\begin{cases}
x~&\text{if}~x_1\approx 1\\
x~&\text{if}~x_1\approx 0~\text{and}~x_1<c\\
(1,x_2+e)~&\text{if}~x_1\approx 0~\text{and}~x_1>c.\\
\end{cases}
\end{align*}
Using such $\psi_c$ and $\phi_{c,e}$, we construct $f$ by sequentially separating sets in $\mcG$ based on their first coordinate.
First, we apply some invertible affine transformation so that the first coordinate of all sets in $\mcG$ is close to zero (as in the left of \cref{fig:encoder3}).
We then split the sets of the largest first coordinate using $\psi_c$ with some proper choice of $c$.
After that, we sequentially split sets as in \cref{fig:encoder3}. Lastly, we apply a projection onto the second coordinate. 
For a more formal argument, see \cref{sec:pflem:inductive-encoder}.
\vspace{-0.1in}

\section{Conclusion}\label{sec:conclusion}
In this work, we characterize the minimum width enabling universal approximation of $L^p([0,1]^{d_x},\bbR^{d_y})$. 
In particular, we consider a general class of activation functions, called squashable, whose alternative composition with affine transformations can approximate both the identity function and $\step$ on compact domains.
We show that for networks using a squashable activation function, the minimum width is $\max\{d_x,d_y,2\}$ unless $d_x=d_y=1$; the same minimum width holds for $d_x=d_y=1$ if the squashable activation function is monotone. 
Since all non-affine analytic functions and a class of piecewise functions are squashable, our result covers almost all practical activation functions.
We believe that our approach would contribute to a better understanding of the expressive power of deep and narrow networks.

\section*{Impact Statement}
This paper investigates the theoretical properties of neural networks on the minimum width enabling universal approximation.
We could not find notable potential societal consequences of our work.

\bibliography{reference-a1}
\bibliographystyle{plainnat}

\newpage
\appendix
\onecolumn

\section{On activation functions}
\subsection{Definition of activation functions}\label{sec:activations}
\begin{itemize}
    \item $\exp$:
    \begin{align*}
        \exp(x) = e^x.
    \end{align*}
    
    \item $\sigmoid$:
    \begin{align*}
        \sigmoid(x) = \frac{1}{1+\exp(-x)}.
    \end{align*}
    
    \item $\tanh$:
    \begin{align*}
        \tanh (x) = \frac{\exp(x)-\exp(-x)}{\exp(x)+\exp(-x)}.
    \end{align*}
    
    \item $\lrelu$: for $\alpha\in(0,1)$
    \begin{align*}
        \lrelu(x;\alpha) = \begin{cases}
            x& \text{if }x>0\\
            \alpha x & \text{if }x\le 0.
        \end{cases}
    \end{align*}
    
    \item $\elu$: for $\alpha>0$
    \begin{align*}
        \elu(x;\alpha) = \begin{cases}
            x& \text{if }x>0\\
            \alpha (\exp(x)-1) & \text{if }x\le 0.
        \end{cases}
    \end{align*}
    
    \item $\selu$: for $\lambda>1$ and $\alpha>0$,
    \begin{align*}
        \selu(x;\lambda, \alpha) = \lambda \times \begin{cases}
            x& \text{if }x>0\\
            \alpha (\exp(x)-1) & \text{if }x\le 0.
        \end{cases}
    \end{align*}

    \item $\gelu$: 
    \begin{align*}
        \gelu(x) = x \times \Phi(x)
    \end{align*}
    where $\Phi$ is the cumulative distribution function for the standard normal distribution.
    
    \item $\celu$: for $\alpha>0$
    \begin{align*}
        \celu(x;\alpha) = \begin{cases}
            x& \text{if }x>0\\
            \alpha (\exp(x/\alpha)-1) & \text{if }x\le 0.
        \end{cases}
    \end{align*}
    
    \item $\softplus$: for $\alpha>0$,
    \begin{align*}
        \softplus(x;\alpha) = \frac{1}{\alpha}\log(1+\exp(\alpha x)).
    \end{align*}
    \item $\swish$:
    \begin{align*}
        \swish(x) = x\times\sigmoid(x).
    \end{align*}
    
    \item $\mish$:
    \begin{align*}
        \mish(x) = x\times\tanh(\softplus(x;1)).
    \end{align*}

    \item $\hswish$:
    \begin{align*}
        \hswish(x) = \begin{cases}
        0&\text{if }x\le -3\\
        x&\text{if } x\ge 3\\
        x(x+3)/6& \text{ otherwise}.
        \end{cases}
    \end{align*}
    
\end{itemize}

\subsection{Proofs related to squashable activation functions}
In this section, we prove \cref{lem:squashable-passingline,lem:analytic,lem:pieceanalytic} by constructing $\sigma$ network of width $1$ satisfying the conditions listed in \cref{cond:step} where $\sigma$ has the property in each lemma.
\subsubsection{Proof of \cref{lem:squashable-passingline}}\label{sec:pflem:passingline}
In this section, we prove \cref{lem:squashable-passingline}. We first prove that if $\sigma$ satisfies the conditions listed in \cref{lem:squashable-passingline}, then $\sigma$ is squashable by explicitly constructing a network of width $1$ satisfying the \cref{cond:step} using the activation $\sigma$ that satisfies the conditions listed in \cref{lem:squashable-passingline}. Namely, we now show that for any $\varepsilon, \zeta>0$ and compact set $\mcK$, there exists a $\sigma$ network $f:\bbR\to\bbR$ of width $1$ such that $|f(x)-\step(x)|<\varepsilon$ for all $x\in\mcK\setminus(-\zeta, \zeta)$.
To this end, without loss of generality, we assume that $c=0$, $\phi(x) = x$ and $\mcK = [-M, M]$ for some $M>0$ and $[-M, M]\subset[a,b]$.

Then, we have $\rho([a,b])\subset[a,b]$.
For any $n\in\bbN$, define $\psi_n:\bbR\to\bbR$ by
\begin{align*}
\psi_n(x) = \rho^n(x).
\end{align*}

Then, $\psi_n([a,b])\subset[a,b]$ and $\psi_n$ is strictly increasing on $[a,b]$. Furthermore, for any $n\in\bbN$,
$\psi_n(x)<\psi_{n+1}(x)$ for $x\in(0,b)$ and $\psi_n(x)>\psi_{n+1}(x)$ for $x\in(a,0)$. We now show that there exists $N\in\bbN$ such that if $n\ge N$,
\begin{align}
 a<\psi_n(-\zeta)<a+(b-a)\varepsilon, \quad b-(b-a)\varepsilon<\psi_n(\zeta)<b.\label{eq:pflem:squashable3}
\end{align}
Then, since $\psi_n$ is strictly increasing, $\psi(x)\in(a,a+(b-a)\varepsilon)$ for any $[-M, -\zeta]$ and $\psi(x)\in(b-(b-a)\varepsilon, b)$ for any $x\in[\zeta, M]$. 
Then, define a $\sigma$ network $f:\bbR\to\bbR$ of width $1$ by
\begin{align*}
f(x) = \frac{1}{b-a}(\psi_N(x)-a).
\end{align*}

Then, $f([-M, M])\subset f([a,b])\subset[0,1]$
and $f$ is strictly increasing, and $0<f(x)\le f(-\zeta)<\varepsilon$ for $x\in[-M, -\zeta]$ and $1-\varepsilon<f(\zeta)\le f(x)<1$ for $x\in[\zeta, M]$. It implies that $f$ is squashable and this completes the proof.

We now show the existence of $N\in\bbN$ such that $\psi_n$ satisfies \cref{eq:pflem:squashable3} if $n\ge N$. Let $a_n = \psi_n(\zeta)$. Then, $a_n<a_{n+1}< b$ for all $n\in\bbN$. Then, by the monotone convergence theorem, there exists $L\in\bbR$ such that $a<L\le b$ and $\lim_{n\to\infty}a_n = L$. Here, if $L<b$, then 
\begin{align*}
    \lim_{n\to\infty}a_{n+1} = \lim_{n\to\infty}\rho(a_n) = \rho(L)>L
\end{align*}
which is a contradiction. Hence, $L = b$ and this guarantees the existence of $N_1\in\bbN$ such that if $n\ge N_1$, then $b-(b-a)\varepsilon<\psi_n(\zeta)<b$. Likewise, there exists $N_2\in\bbN$ such that if $n\ge N_2$, then $a<\psi_n(-\zeta)<a+(b-a)\varepsilon$. If we choose $N>\max\{N_1, N_2\}$, then our $\sigma$ network $f$ of width $1$ satisfies \cref{cond:step}.

\subsubsection{Proof of \cref{lem:analytic}}\label{sec:pflem:analytic}
In this section, we prove \cref{lem:analytic}. To this end, it suffices to show the existence of the $\sigma$ network $\rho:\bbR\to\bbR$ of width $1$ such that 
\begin{itemize}
    \item $\rho$ is strictly increasing on $[0,1]$,
    \item $\rho(0) = 0$ and $\rho(1)=1$, and
    \item $\rho'(0)<1$ and $\rho'(1)<1$.
\end{itemize}
Then, from the second and third line in the above conditions, one can observe that $\rho(x)<x$ if $x\in(0,\delta)$ and $\rho(x)>x$ if $x\in(1-\delta,1)$ for some $\delta>0$. Then, by the intermediate value theorem, the equation $\rho(x)=x$ has at least one solution in (0,1). Here, since $\rho$ is analytic, there are finitely many solutions $c_1, \cdots, c_k\in(0,1)$ such that $c_1<\cdots<c_k$ and $\rho(c_i) = c_i$ for $i\in[k]$. If $k=1$, then $\rho$ satisfies the conditions of \cref{lem:squashable-passingline} with $[0,1]$ and $\phi(x) = x$. Otherwise, $\rho$ satisfies the conditions of \cref{lem:squashable-passingline} with $[0,c_2]$ and $\phi(x) = x$. It completes the proof.

We now construct such a $\sigma$ network $\rho$ by considering the following cases: (1) there exists $a\in\bbR$ such that $\sigma'(a) = 0$ and (2) $\sigma'(x)\ne 0$ for all $x\in\bbR$. 

We considered the case (1) in \cref{lem:squashable-critical} in \cref{sec:additionalcondition}. 
We now consider the case (2): $\sigma'(x)\ne 0$ for all $x\in\bbR$. Without loss of generality, $\sigma'(x)>0$ for all $x\in\bbR$. To this end, we consider the following cases again: (2-1) there exists $c\in\bbR$ such that $\sigma''(x)>0$ in $(c-\delta, c)$ and $\sigma''(x)<0$ in $(c,c+\delta)$ for some $\delta>0$ and (2-2) otherwise. 

We considered the case (2-1) in \cref{lem:squashable-inflection} in \cref{sec:additionalcondition}.
We now consider the case (2-2). Specifically, it suffices to consider the case that there exists $a\in\bbR$ such that $\sigma''(a)>0$ and $\sigma''(x)\ge 0$ for $x>a$. Otherwise, suppose that $\sigma''(x)\le 0$ for all $x\in\bbR$. Then, we can makes $\sigma$ to convex function by taking an affine transformation: $\sigma_0(x) = -\sigma(-x)$.

Without loss of generality, assume that $a=0$ and $\sigma(0) = 0$. Then, we define a $\sigma$ network $\psi:\bbR\to\bbR$ such that 
\begin{align*}
    \psi(x)= \frac{1}{\sigma(b)}\sigma(bx)
\end{align*}
for $b>0$. We will assign an explicit value of $b$ later. Then, we have $\psi(0)=0$, $\psi(1) = 1$, and $\psi$ is strictly increasing on $[0,1]$. Then, we construct a $\sigma$ network $\rho:\bbR\to\bbR$ of width 1 by
\begin{align*}
    \rho(x) = 1-\psi(1-\psi(x)).
\end{align*}
Then, $\rho(0)=0, \rho(1) = 1$, and $\rho$ is strictly increasing on $[0,1]$. Furthermore, one can observe that 
\begin{align*}
    \rho'(0)=\rho'(1) = \psi'(0)\psi'(1) = \frac{b^2\sigma'(b)\sigma'(0)}{\sigma(b)^2}.
\end{align*}
We now show the existence of $b\in\bbR$ such that 
\begin{align*}
    \frac{b^2\sigma'(b)\sigma'(0)}{\sigma(b)^2}<1.
\end{align*}
To this end, consider a function $g:(0,\infty)\to\bbR$ defined by 
\begin{align*}
    g(x) = \frac{1}{x}-\frac{\sigma'(0)}{\sigma(x)}.
\end{align*}
Then, one can observe that
\begin{align*}
    g'(x) = \frac{1}{x^2}\left(\frac{x^2\sigma'(x)\sigma'(0)}{\sigma(x)^2}-1\right).
\end{align*}
Since $x>0$, it suffices to show the existence of $b>0$ such that $g'(b)<0$.
Since $\sigma''(0)>0$ and $\sigma(x)>0$ for all $x>0$, it can be easily shown that $\sigma(x)>\sigma'(0)x$ for all $x>0$. It implies that $g(x)>0$ for $x>0$. Furthermore, since $\sigma(x)\to\infty$ as $x\to\infty$, it holds that $g(x)\to 0$ as $x\to\infty$. Then, there exists $M>1$ such that $g(1)>g(M)$ since $g(x)\to 0$ as $x\to\infty$ and $g(1)>0$. Then, by the mean value theorem, there exists $b\in(1,M)$ such that
\begin{align*}
    \frac{g(M)-g(1)}{M-1} = g'(b)<0.
\end{align*}
It completes the proof.

\subsubsection{Proof of \cref{lem:pieceanalytic}}\label{sec:pflem:pieceanalytic}
In this section, we prove \cref{lem:pieceanalytic}. To this end, we first consider the case that the given activation is a piecewise linear function. Without loss of generality, we assume that 
\begin{align}
\sigma_1(x) = \begin{cases}
ax&x\in[-1,0)\\
x&x\in [0,2]
\end{cases}\label{eq:pflem:pieceanalytic}
\end{align}
where $0<a<1$. We now construct a $\sigma$ network $\rho$ of width 1 as
\begin{align*}
\rho(x) = 1-\sigma_1(1-\sigma_1(x))= \begin{cases}
ax&x\in[-1,0)\\
x&x\in[0,1)\\
ax+1-a&x\in[1,2].
\end{cases}
\end{align*} 
Since $0<a<1$, it is easy to observe that $\sigma_1$ satisfies \cref{cond:step} by \cref{lem:squashable-passingline}. 

We now consider the general case. Suppose that $\sigma:\bbR\to\bbR$ satisfies the conditions listed in \cref{lem:pieceanalytic}. We show this by constructing a $\sigma$ network $\psi$ of width $1$ that approximates $\sigma_1(x)$ in \cref{eq:pflem:pieceanalytic} with $a = \sigma'(c_-)/\sigma'(c_+)$ within an arbitrary error for any $x\in[-1,2]$. Then, we can easily verify that \cref{lem:squashable-passingline} can be applied to the same construction of $\sigma$ network of width $1$ as above, $1-\psi(1-\psi(x))$, and this completes the proof.

We now show the existence of such $\psi$. To this end, without loss of generality, we assume that $c=0, \sigma(0)=0$, $0<\sigma'(c_{-})<\sigma'(c_+)$, and $\sigma$ is strictly increasing on $(c-\delta, c+\delta)$. For $r>0$, construct a $\sigma$ network $\psi$ of width $1$ as
\begin{align*}
\psi_r(x) = \frac{\sigma(rx)}{r}.
\end{align*}
By the mean value theorem, for $-1\le x<0$, there exists $d_r\in(rx,0)$ such that $\psi_r(x) = x\sigma'(d_r)$ and for $0<x\le 2$, there exists $e_r \in(0,rx)$ such that $\psi_r(x) = x\sigma'(e_r)$. Since $\sigma'(x)$ is continuous on $(c-\delta, c+\delta)$, it holds that $\sigma'(d_r)\to \sigma'(c_{-})$ and $\sigma'(e_r)\to\sigma'(c_{+})$ as $r\to 0$, respectively. It implies that
\begin{align*}
\lim_{r\to 0}\psi_r(x) = \begin{cases}
\sigma'(c_-)x&x\in[-1,0)\\
\sigma'(c_+)x&x\in[0,2].
\end{cases}
\end{align*}
Thus, choosing $\psi(x) = \psi_r(x)/\sigma'(c_+)$ with sufficiently small $r>0$ completes the proof.

\subsection{Additional properties for functions to satisfy \cref{cond:step}}\label{sec:additionalcondition}

In this section, we suggest the additional properties for activation functions to satisfy \cref{cond:step}. \cref{lem:squashable-inflection} implies that an activation $\sigma$ satisfies \cref{cond:step} if there exists a point where the sign of $\sigma''$ converts from positive to negative.
\begin{lemma}\label{lem:squashable-inflection}
Let $c\in\bbR$ and $\delta>0$. Suppose that a function $\sigma:\bbR\to\bbR$ such that $\sigma$ is twice differentiable in $(c-\delta, c+\delta)$, $\sigma''(x)>0$ in $(c-\delta, c)$ and $\sigma''(x)<0$ in $(c, c+\delta)$. Then, $\sigma$ satisfies \cref{cond:step}.
\end{lemma}
\begin{proof}
To prove \cref{lem:squashable-inflection}, we now choose appropriate $a,b\in\bbR$ and $\phi:\bbR\to\bbR$ and apply \cref{lem:squashable-passingline} with our $a, b, c$ and $\phi$. We consider a line passing $(c,\sigma(c))$ as $\phi$. Since $\rho''(x)>0$ if $x<c$ and $\rho''(x)<0$ if $x>c$, we can choose a slope of $\phi$ so that $\phi$ and $\rho$ meet once in $(c-\delta, c)$ and $(c, c+\delta)$, respectively.
Let $\alpha = \max\{\sigma(c)-\sigma(c-\delta/2), \sigma(c+\delta/2)-\sigma(c)\}$ and $\phi(x) = \frac{\alpha}{\delta/2}(x-c)+\sigma(c)$. Here, one can easily observe that $\frac{\alpha}{\delta/2}<\sigma'(c)$. Without loss of generality, suppose that $\alpha = \sigma(c)-\sigma(c-\delta/2)$. Then, it holds that
\begin{align*}
\phi(c+\delta/2) = \sigma(c) - \sigma(c-\delta/2)+\sigma(c)\ge \sigma(c+\delta/2) - \sigma(c) = \sigma(c+\delta/2).
\end{align*}
Then, by the intermediate value theorem, there exists $b\in(c, c+\delta/2]$ such that $\phi(b) = \sigma(b)$. Furthermore, since $\phi(c-\delta/2) = \sigma(c-\delta/2)$, choosing $a = c-\delta/2$ and applying \cref{lem:squashable-passingline} with our $a, b, c$ and $\phi$ completes the proof.
\end{proof}

\cref{lem:squashable-critical,lem:squashable-quadratic} imply that if $\sigma$ satisfies a condition stronger than the analytic condition in a compact interval, then $\sigma$ satisfies \cref{cond:step}.

\begin{lemma}\label{lem:squashable-critical}
Consider $a_1, a_2\in\bbR$ such that $\sigma(x)$ is nonaffine analytic on $x\in[a_1, a_2]$. Suppose that there exists $c\in[a_1, a_2]$ such that $\sigma'(c)=0$. Then, $\sigma$ satisfies \cref{cond:step}.
\end{lemma}
\begin{proof}
It suffices to show the existence of the $\sigma$ network $\rho:\bbR\to\bbR$ of width $1$ such that $\rho$ is strictly increasing on $[0,1]$, $\rho(0)=0, \rho(1)=1,\rho'(0)<1$ and $\rho'(1)<1$ (see \cref{sec:pflem:analytic}).
Since $\sigma$ is a nonaffine analytic function that has a zero derivative at some point, $b\in(c, a_2]$ such that 
$\sigma$ is strictly monotone on $[c,b]$ with nonlinearity.
Without loss of generality, assume that $c=0$, $\sigma(0)=0$ and $\sigma(x)$ is strictly increasing on $[0,b]$. Then, we define a $\sigma$ network $\psi:\bbR\to\bbR$ such that 
\begin{align*}
\psi(x) = \frac{1}{\sigma(b)}\sigma(bx).
\end{align*}
Then, $\psi(0) = 0$, $\psi(1) = 1$, and $\psi$ is strictly increasing on $[0,1]$. We now construct a $\sigma$ network $\rho$ by
\begin{align*}
\rho(x) = 1-\psi(1-\psi(x)).
\end{align*}
Then, $\rho(0) = 0$, $\rho(1) = 1$, and $\rho$ is strictly increasing on $[0,1]$. Furthermore, one can observe that 
\begin{align*}
    \rho'(x) = \psi'(1-\psi(x))\psi'(x).
\end{align*}
Then, we have $\rho'(0)=\rho'(1)=0$ since $\psi'(0)=0$. It completes the proof.
\end{proof}

\begin{lemma}\label{lem:squashable-quadratic}
Consider $a_1,a_2\in \mathbb{R}$ such that $\sigma(x)$ is analytic on $x\in [a_1, a_2]$.
Assume that there exists $x\in [a_1, a_2]$ such that
\begin{align*}
a_2  \ge \frac{2\sigma'(x)}{\sigma''(x)} + x.
\end{align*}
Then, $\sigma$ satisfies \cref{cond:step}.
\end{lemma}

\begin{proof}
In this proof, $\sigma^{(n)}(x)$ is defined as $n$-times derivative: $\sigma^{(n)}(x) = \frac{d^n\sigma(x)}{dx^n}$.
We only need to consider the case $\sigma'(x) >0$ and $\sigma^{(2)}(x) > 0$; see the case (1) and (2-1) in \cref{sec:pflem:analytic}.

Consider an arbitrary $x_0\in (a_1, a_2)$.
For $b\in (a_1 - x_0, a_2 - x_0)$, define $\psi: (0-\epsilon,1+\epsilon)\rightarrow \mathbb{R}$ as 
\begin{align*}
    \psi(x) \coloneqq \frac{1}{\sigma(b + x_0) - \sigma(x_0)} (\sigma(bx+ x_0) - \sigma(x_0)).
\end{align*}
Then, $\psi(0) = \psi(1)  =1$.
Define $\rho$ as 
\begin{align*}
    \rho(x) \coloneqq 1-\psi(1-\psi(x)).
\end{align*}
Then, 
\begin{align*}
    \rho'(0)=\rho'(1) = \psi'(0)\psi'(1) = \frac{b^2\sigma'(b+x_0)\sigma'(x_0)}{(\sigma(b+x_0)-\sigma(x_0))^2}.
\end{align*}
It is sufficient to find a value $b$ such that $\rho'(0)=\rho'(1)<1$.
Define $g$ as
\begin{align*}
    g(x) \coloneqq \frac{1}{x}-\frac{\sigma'(x_0)}{\sigma(x + x_0) - \sigma(x_0)}.
\end{align*}
Then, as 
\begin{align*}
    g'(x) = -\frac{1}{x ^2} +\left(\frac{\sigma'(x_0)\sigma'(x+x_0)}{(\sigma(x + x_0) - \sigma(x_0))^2}\right)
    = \frac{1}{x^2}    \left(\frac{x^2\sigma'(x+x_0)\sigma'(x_0)}{(\sigma(x + x_0) - \sigma(x_0))^2}-1\right),
\end{align*}
it is sufficient to find a number $x$ such that $g'(x)<0$. 
Then, there exist smooth functions $h,h_1,h_2$ such that
\begin{align*}
    g(x) &= \frac{1}{x} -\frac{\sigma'(x_0)}{\sigma(x + x_0)-\sigma(x_0)}
    = \frac{1}{x} -\frac{\sigma'(x_0)}{\sigma'(x_0)x + \sigma^{(2)}(x_0)\frac{x^2}{2} + \sigma^{(3)}(x_0)\frac{x^3}{6} + x^4h(x)}
    \\ &= \frac{1}{x} -\frac{1}{x + \frac{\sigma^{(2)}(x_0)}{\sigma'(x_0)}\frac{x^2}{2} + \frac{\sigma^{(3)}(x_0)}{\sigma'(x_0)}\frac{x^3}{6} + \frac{h(x)}{\sigma'(x_0)}x^4}
    \\ &= \frac{ \frac{\sigma^{(2)}(x_0)}{2\sigma'(x_0)}\left(1  + \frac{\sigma^{(3)}(x_0)}{\sigma^{(2)}(x_0)}\frac{x}{3} + \frac{h(x)}{\sigma^{(2)}(x_0)}x^2 \right)}{1 + \frac{\sigma^{(2)}(x_0)}{\sigma'(x_0)}\frac{x}{2} + \frac{\sigma^{(3)}(x_0)}{\sigma'(x_0)}\frac{x^2}{6} + \frac{h(x)}{\sigma'(x_0)}x^3}
     = \frac{\sigma^{(2)}(x_0)}{2\sigma'(x_0)} \frac{1  + \frac{\sigma^{(3)}(x_0)}{\sigma^{(2)}(x_0)}\frac{x}{3} + h_2(x)x^2}{1 + \frac{\sigma^{(2)}(x_0)}{\sigma'(x_0)}\frac{x}{2} + h_1(x)x^2}.
\end{align*}
Then, $g'(x)<0$ if 
\begin{equation}
    \frac{\sigma^{(2)}(x_0)}{2\sigma'(x_0)} >  \frac{\sigma^{(3)}(x_0)}{2\sigma^{(2)}(x_0)}.
\end{equation}
Assume that the above inequality is not satisfied for any $x_0\in (a_1, a_2)$; that is, for any $x\in (a_1, a_2)$
\begin{align*}
    \frac{\sigma^{(2)}(x)}{2\sigma'(x)} \le  \frac{\sigma^{(3)}(x)}{3\sigma^{(2)}(x)}.
\end{align*}
Then, for any $a_1< x_1<y<a_2$,
\begin{align*}
    &\int_{x_1}^y\frac{\sigma^{(2)}(x)}{2\sigma'(x)}dx\le  \int_{x_1}^y\frac{\sigma^{(3)}(x)}{3\sigma^{(2)}(x)}dx \iff \frac{3}{2}\log\left(\frac{\sigma'(y)}{\sigma'(x_1)}\right)\le \log\left(\frac{\sigma^{(2)}(y)}{\sigma^{(2)}(x_1)}\right)
     \\ &\iff \frac{\sigma^{(2)}(x_1)}{\sigma'(x_1)^{\frac{3}{2}}} \le \left(\frac{\sigma^{(2)}(y)}{\sigma'(y)^{\frac{3}{2}}}\right)
    ,
\end{align*}
which leads to
\begin{align*}
    \frac{\sigma^{(2)}(x_1)}{\sigma'(x_1)^{\frac{3}{2}}} (z-x_1)\le 2\left(\frac{1}{\sigma'(x_1)^{\frac{1}{2}}} -\frac{1}{\sigma'(z)^{\frac{1}{2}}}\right),
\end{align*}
for any $a_1<x_1<z<a_2$.
Thus,
\begin{align*}
    \frac{1}{\left(\frac{1}{\sigma'(x_1)^{\frac{1}{2}}} -  \frac{\sigma^{(2)}(x_1)}{2\sigma'(x_1)^{\frac{3}{2}}} (z-x_1)\right)^2} \le \sigma'(z).
\end{align*}
\end{proof}

We lastly present \cref{lem:squashable-limit} which implies that if strictly monotone $\sigma$ has a limit, then $\sigma$ satisfies \cref{cond:step}.

\begin{lemma}\label{lem:squashable-limit}
A continuous function $\sigma:\bbR\to\bbR$ satisfies \cref{cond:step} if $\sigma$ has strictly monotonicity and there exists $\lim_{x\to\infty}\sigma(x)$ or $\lim_{x\to-\infty}\sigma(x)$.
\end{lemma}
\begin{proof}
Without loss of generality, we assume that $\sigma(x)$ is strictly increasing and $\lim_{x\to-\infty}\sigma(x)=0$. We consider the two cases: (1) $\lim_{x\to\infty}\sigma(x) = \alpha<\infty$, and (2) $\lim_{x\to\infty}\sigma(x) = \infty$. 

For the first case, we can easily verify that $\sigma$ satisfies \cref{cond:step} by composing affine functions before and after $\sigma$:
\begin{align*}
\rho(x) = \frac{1}{\alpha}\sigma(Mx)
\end{align*}
where $M>0$ is sufficiently large.

We now consider the second case. Suppose that $\lim_{x\to\infty}\sigma(x) = \infty$.
We construct a $\sigma$ network $\psi$ of width $1$ such that 
\begin{align*}
\psi(x) = \frac{1}{\sigma(1)}\times \left(\sigma(1) - \sigma(1 - \sigma(x))\right).
\end{align*}
Then, it is easy to observe that $\psi$ is strictly increasing, $\lim_{x\to\infty}\psi(x) = 1$ and $\lim_{x\to-\infty}\psi(x) = 0$. Then, we can consider $\phi$ as in the first case. Hence, $\sigma$ is squashable and this completes the proof.
\end{proof}

\vfill

\section{Our choice of $\delta, \gamma,N$}\label{sec:error-analysis}
We first choose a sufficiently small $\delta>0$ so that $\delta \le \varepsilon/(d_y^{1/p} \times 3^{1+1/p})$.
And then, choose a small enough $\gamma > 0$ so that $\gamma \le \varepsilon^p/(3d_x d_y)$ and $\omega_{f_{\mathrm{dec}}}(\gamma) \le \varepsilon/3^{1+1/p}$.
Lastly, we choose large enough $N \in \bbN$ satisfying $\diam(f^*(\mcT_\nu)) = \omega_{f^*}((1-2\gamma)/N) \le \varepsilon/(d_y^{1/p} \times 3^{1+1/p})$ for each $\nu \in [N]^{d_x}$.
Here, $\omega_{f_{\mathrm{dec}}}$ and $\omega_{f^*}$ denote the modulus of continuity of given function $f$ in the $p$-norm: $\|f(x) - f(x')\|_p \le \omega_f (\|x-x'\|_p)$ for all $x , x' \in [0,1]^{d_x}$.
Then,
\begin{align*}
\|f_\mathrm{dec}\circ f_\mathrm{enc} - f^*\|_{L^p}^p &= \int_{[0,1]^{d_x}}\|f_\mathrm{dec}\circ f_\mathrm{enc}(x)-f^*(x)\|_p^p d\mu_{d_x}\\
&\le \int_{[0,1]^{d_x}\setminus\bigcup_{\nu\in[N]^{d_x}}\mcT_\nu}\|f_\mathrm{dec}\circ f_\mathrm{enc}(x)-f^*(x)\|_p^p d\mu_{d_x} \\
& \qquad \qquad \qquad \qquad \qquad \qquad + \int_{\bigcup_{\nu\in[N]^{d_x}}\mcT_\nu}\|f_\mathrm{dec}\circ f_\mathrm{enc}(x)-f^*(x)\|_p^p d\mu_{d_x}\\
&\le d_y \times  \mu_{d_x}\left([0,1]^{d_x}\setminus\bigcup_{\nu\in[N]^{d_x}} \mcT_\nu \right) + \sum_{\nu\in[N]^{d_x}}\int_{\mcT_\nu}\|f_\mathrm{dec}\circ f_\mathrm{enc}(x)-f^*(x)\|_p^p d\mu_{d_x}\\
&\le d_y \times (1-(1-2\gamma)^{d_x}) \\
& \qquad \qquad  + \sum_{\nu\in[N]^{d_x}}\int_{\mcT_\nu}(\|f_\mathrm{dec}\circ f_\mathrm{enc}(x)-f_\mathrm{dec}(c_\nu)\|_p + \|f_\mathrm{dec}(c_\nu)-f^*(x)\|_p )^pd\mu_{d_x}\\
&\le  2 d_x d_y \gamma + \sum_{\nu\in[N]^{d_x}}\int_{\mcT_\nu}(\omega_{f_{\mathrm{dec}}}(\gamma) + d_y^{1/p} \times (\diam(f^*(\mcT_\nu)) +\delta ) )^pd\mu_{d_x}\\
&\le 2 d_x d_y \gamma + (\omega_{f_{\mathrm{dec}}}(\gamma) + d_y^{1/p} \times (\diam(f^*(\mcT_\nu)) +\delta ) )^p \le \varepsilon^p
\end{align*}
where $c_\nu$ is chosen so that $\dist{f_\mathrm{dec}(c_\nu)}{f^\ast(\mcT_\nu)}\le\delta$ for each $\nu\in[N]^{d_x}$.
This leads us to the statement of \cref{lem:ub}.

\vfill

\newpage
\section{Proof of \cref{lem:simplification}}\label{sec:pflem:simplification}
In this section, we prove \cref{lem:simplification}.
Since $f:\mcK\to\bbR^{d_y}$ is a $(\sigma,\iota)$ network of width $w$, we can express $f:\mcK\to\bbR^{d}$ as follows:
\begin{align*}
f = t_L\circ \phi_{L-1}\circ t_{L-1} \circ \cdots\circ \phi_{1} \circ t_1
\end{align*}
where $t_\ell:\bbR^{d_{\ell-1}}\to \bbR^{d_\ell}$ is an affine transformation, and $\phi_\ell(x) = (\rho_{\ell, 1}(x), \cdots, \rho_{\ell, d_\ell}(x))$ for $\rho_{\ell, 1}, \cdots, \rho_{\ell, d_\ell}\in\{\sigma, \iota\}$ for all $\ell\in[L]$. Since $\sigma$ satisfies \cref{cond:id}, by \cref{lem:kidger}, for arbitrary compact set $\mcC$ and for any $\delta>0$, there exist affine transformations $h_1:\bbR\to\bbR$ and $h_2:\bbR\to\bbR$ such that 
\begin{align*}
|h_1\circ \sigma \circ h_2 (x) - \iota(x)|<\delta
\end{align*}

for all $x\in\mcC$; we will assign explicit value to $\delta$ later. We denote $h_1\circ \sigma\circ h_2$ as $\tilde\sigma$. We note that this lemma can be applied for any given compact set. Since we are considering a compact domain and a continuous activation function, the error arising from replacing $\iota$ with $\tilde{\sigma}$ can be reduced.

To this end, we choose a $\sigma$ network $g$ by applying same affine transformation $t_1, \cdots, t_L$ and $\tilde\sigma$:
\begin{align*}
g = t_L\circ \psi_{L-1}\circ t_{L-1} \circ \cdots\circ \psi_{1} \circ t_1
\end{align*}
where $\psi(x) = (\tilde\rho_{\ell,1}(x),\cdots,\tilde\rho_{\ell, d_\ell}(x))$ with $\tilde\rho_{\ell,i} = \sigma$ if $\rho_{\ell,i}=\sigma$ and $\tilde\rho_{\ell,i}=\tilde\sigma$ if $\rho_{\ell,i}=\iota$ for $\ell\in[L]$ and $i\in[d_\ell]$.

We denote $f_\ell$ and $g_\ell$ by the first $\ell$ layers of $f$ and $g$ with the subsequent affine transformation $t_\ell$, respectively. i.e.,
\begin{align*}
    f_\ell =  t_\ell\circ \phi_{\ell-1}\circ t_{\ell-1} \circ \cdots\circ \phi_{1} \circ t_1\quad\text{and}\quad g_\ell = t_\ell\circ \psi_{\ell-1}\circ t_{\ell-1} \circ \cdots\circ \psi_{1} \circ t_1.
\end{align*}
Then, for each $\ell\in[L]\setminus\{1\}$ and for any $x\in\mcK$, it holds that
\begin{align*}
\|f_\ell(x)-g_\ell(x)\|_\infty &= \|t_\ell\circ \phi_{\ell-1}\circ f_{\ell-1}(x) - t_\ell\circ \psi_{\ell-1}\circ g_{\ell-1}(x)\|_\infty\\
&\le \omega_{t_\ell}(\|\phi_{\ell-1}\circ f_{\ell-1}(x) - \psi_{\ell-1}\circ g_{\ell-1}(x)\|_\infty)\\
&\le \omega_{t_\ell}(\|\phi_{\ell-1}\circ f_{\ell-1}(x) - \phi_{\ell-1}\circ g_{\ell-1}(x)\|_\infty+\|\phi_{\ell-1}\circ g_{\ell-1}(x) - \psi_{\ell-1}\circ g_{\ell-1}(x)\|_\infty).
\end{align*}
Here, we note that for any $\ell\in[L]$, $\omega_{t_\ell}$ is well-defined since $t_\ell$ is uniformly continuous on $\bbR^{d_{\ell-1}}$. Then, by the definition of $\psi_{\ell-1}$ and $\tilde\sigma$, it holds that
\begin{align*}
\|\phi_{\ell-1}\circ g_{\ell-1}(x) - \psi_{\ell-1}\circ g_{\ell-1}(x)\|_\infty \le \max_{i\in[d_{\ell-1}]}|\tilde\sigma(g_{\ell-1}(x)_i) - \iota(g_{\ell-1}(x)_i)|<\delta.
\end{align*}
Furthermore, since we are considering the compact domain and $\phi_{\ell-1}$ is continuous, $\omega_{\phi_{\ell-1}}$ is well-defined and 
\begin{align*}
\|\phi_{\ell-1}\circ f_{\ell-1}(x) - \phi_{\ell-1}\circ g_{\ell-1}(x)\|_\infty \le \omega_{\phi_{\ell-1}}(\|f_{\ell-1}(x)-g_{\ell-1}(x)\|_\infty)
\end{align*}
Hence, we have
\begin{align}
\|f_\ell(x)-g_\ell(x)\|_\infty &=\omega_{t_\ell}(\|\phi_{\ell-1}\circ f_{\ell-1}(x) - \phi_{\ell-1}\circ g_{\ell-1}(x)\|_\infty+\|\phi_{\ell-1}\circ g_{\ell-1}(x) - \psi_{\ell-1}\circ g_{\ell-1}(x)\|_\infty)\nonumber\\
&\le \omega_{t_\ell}(\omega_{\phi_{\ell-1}}(\|f_{\ell-1}(x)-g_{\ell-1}(x)\|_\infty)+\delta)\label{eq:pflem:simplification}
\end{align}
for all $\ell\in[L]\setminus\{1\}$. By iteratively applying \cref{eq:pflem:simplification}, we have
\begin{align*}
\|f(x)-g(x)\|_\infty&\le \omega_{t_L}(\omega_{\phi_{L-1}}(\|f_{L-1}(x) - g_{L-1}(x)\|_\infty)+\delta)\\
&\vdots\\
&\le \omega_{t_L}(\omega_{\phi_{L-1}}(\cdots (\omega_{t_3}(\omega_{\phi_2}(\omega_{t_2}(\delta)+\delta)+\delta)+\delta)\cdots)+\delta).
\end{align*}
Consequently, by choosing sufficiently small $\delta>0$, we can reduce this within arbitrary error $\varepsilon>0$ and this completes the proof.

\newpage

\section{Proof of \cref{lem:decoder-inductive}}\label{sec:pflem:decoder-inductive}
In this section, we prove \cref{lem:decoder-inductive}. To show \cref{lem:decoder-inductive}, we construct $(\sigma, \iota)$ network $\tilde f$ of width $d+1$ as follows: for each $i\in[d]$ and $\nu\in[N]^d$, 
\begin{align}
\tilde f(x)_i = f(x)_i\quad \text{and}\quad \left[\frac{1}{2N}, 1-\frac{1}{2N}\right]\subset \tilde f(\mcI_\nu)_{d+1}.\label{eq:pflem:decoder-inductive}
\end{align}
Then, since $\tilde f$ is continuous, for each $\nu\in[N]^d$ and $j\in[N]$, there exists $\mcJ_{(\nu, j)}\subset\mcI_\nu$ such that 
\begin{align*}
\tilde f(\mcJ_{(\nu, j)})_{d+1}\subset\left[\frac{j-1}{N},\frac{j}{N} \right].
\end{align*}
Furthermore, since $\mcJ_{(\nu, j)}\subset\mcI_\nu$ for each $\nu = (\nu_1, \cdots, \nu_d)\in[N]^d$ and $j\in[N]$, it can be easily observed that
\begin{align*}
\tilde f(\mcJ_{(\nu, j)})_{i}\subset\left[\frac{\nu_i-1}{N},\frac{\nu_i}{N} \right]
\end{align*}
for all $i\in[d]$. It implies that $\tilde f(\mcJ_{\nu'})\subset\mcC_{N, d+1, \nu'}$ and this completes the proof.

We now construct a $(\sigma, \iota)$ network of width $d+1$ satisfying \cref{eq:pflem:decoder-inductive}. To this end, we first present the following lemma.
\begin{lemma}\label{lem:indicator}
Let $\sigma\in\fkS$ and $z_{1},z_{2}, \cdots, z_{k} \in [0,1]$ such that $z_i\neq z_j$ for all $i\neq j$. Let $\gamma>0$ such that $\gamma<\min_{i\ne j}|z_i-z_j|/2$. Then, there exists a $(\sigma, \iota)$ network $f:[0,1] \to \mathbb{R}^{2}$ of width 2 satisfying the following:
   \begin{itemize}
    \item $f(x)_1=x$ on $[0,1]$,
    \item $\left[\frac{1}{2N}, 1-\frac{1}{2N}\right]\subset f(\mcB_\gamma(z_i))_2$ for all $i \in [k]$,
    \item $f([0,1]) \subset [0,1]^2$.
   \end{itemize}
\end{lemma}

One can observe that \cref{lem:indicator} allows us to prove \cref{lem:decoder-inductive} directly. We choose $\gamma>0$ and $z_\nu\in\mcI_\nu$ for each $\nu$ such that $\mcB_\gamma(z_\nu)\subset\mcI_\nu$. Applying \cref{lem:indicator} with our choices of $z_\nu$'s and $\gamma$, we construct a $(\sigma, \iota)$ network $\phi:[0,1]\to\bbR^2$ of width $2$ satisfying the conditions listed in \cref{lem:indicator}. Then, we complete the proof by constructing $\tilde f$ in \cref{eq:pflem:decoder-inductive} as follows:
\begin{align*}
&\tilde f(x)_1 = \phi(f(x)_1)_1, \quad \tilde f(x)_{d+1} = \phi(f(x)_1)_2, \text{ and}\\
&\tilde f(x)_i = \iota\circ\cdots\circ\iota(f(x)_i)\text{ for all }i\in\{2,\cdots,d\}.
\end{align*}

\subsection{Proof of \cref{lem:indicator}}\label{sec:pflem:indicator}
Without loss of generality, we assume $k=2m$ for some $m \in \bbN$ and $0=z_0<z_1<z_2<\cdots<z_{2m}<z_{2m+1}=1$; otherwise, we can add an auxiliary $z_{k+1}\in\bbR$ such that $z_k<z_{k+1}<1$. Let $\mcX = \{z_1, z_2, \cdots, z_{2m}\}$, $\mcD_{\mcX, \gamma} = [0, 1] \setminus \bigcup_{i=1}^{2m}(z_i-\gamma, z_i+\gamma)$, and $\mathcal{A}_\mcX=\bigcup_{i=1}^{m}(z_{2i-1},z_{2i}]$.

To construct $f$ in \cref{lem:indicator} using $(\sigma, \iota)$ network, we use the \cref{cond:step} that for any compact set $\mcC$, $\sigma$ can approximate $\step$ except for the neighborhood of a breakpoint.
We first construct $(\step, \iota)$ network $h:[0,1]\to\{0, 1\}$ of width $2$ such that 
\begin{align}
h(x) = 
\begin{cases}
1&\text{if } x\in\mcA_{\mcX}\\
0&\text{otherwise}
\end{cases},\label{eq:pflem:indicator}
\end{align}
and then we construct a $(\sigma, \iota)$ network $f:[0,1]\to\bbR^2$ of width $2$ such that $f(x)_1 = x$ and $|f(x)_2-h(x)|<1/2N$ except for the neighborhood of each $z_i\in\mcX$. Since $f$ is a continuous function, one can observe that such $f$ satisfies the conditions listed in \cref{lem:indicator}.

We first construct $h$ in \cref{eq:pflem:indicator} as follows: $h = h_{m+1}$ where $h_{m+1}(x)$ is recursively defined as 
\begin{align}
h_{1}(x) = \step(x-z_{m+1}) && h_{\ell}(x) = \step(x-z_{m-\ell+2} +(z_{m-\ell+2}-z_{m+\ell})h_{\ell-1}(x)).\label{eq:pflem:indicator1}
\end{align}

From \cref{eq:pflem:indicator1},
\begin{align*}
h_{\ell}(x)= \begin{cases}
\step(x-z_{m-\ell+2})&h_{\ell-1}(x)=0\\
\step(x-z_{m+\ell})&h_{\ell-1}(x)=1
\end{cases}
\end{align*}
for any $\ell\in\{2, \cdots, m+1\}$.
One can observe that $h_{\ell}$ forms additional breakpoints $z_{m-\ell+2}$ and $z_{m+\ell}$, and for any $x\in[z_i, z_{i+1})$ where $i\in\{m-\ell+2, \cdots, m+\ell\}$, the values of $h_{\ell}(x)$ alternates with $0$ and $1$ as $\ell$ increases. Hence, $h_{\ell}(x)$ in \cref{eq:pflem:indicator1} can be rewritten by
\begin{align*}
h_{\ell}(x)= \begin{cases}
1& x\in[z_{m-\ell+2k}, z_{m-\ell+2k+1}), ~\forall k\in[\ell-1]~~\text{or}~~x\ge z_{m+\ell}\\
0&\text{otherwise}
\end{cases}
\end{align*}
for any $\ell\in\{2, \cdots, m+1\}$, which implies that $h_{m+1}$ is equal to $h$ in \cref{eq:pflem:indicator}.

We now construct a $(\sigma, \iota)$ network $f$ of width $2$ based on $h$. It suffices to show that for any $\varepsilon>0$ and $\ell\in[m+1]$ there exists a $(\sigma, \iota)$ network $f_{\ell}:[0,1]\to\bbR^2$ 
of width $2$ such that 
\begin{itemize}[leftmargin = 25pt]
\item[C1.] $f_{\ell}(x)_1 = x$ on $[0,1]$,
\item[C2.] $|f_{\ell}(x)_2-h_{\ell}(x)|<\varepsilon$ for $x\in\mcD_{\mcX, \gamma}$,
\item[C3.] $f_{\ell}([0,1]) \subset[0,1]^2$.
\end{itemize}

Then, choosing $f = f_{m+1}$ with $\varepsilon<1/(2N)$ completes the proof: C1 and C3 directly imply the first and third conditions of \cref{lem:indicator}, respectively, and C2 guarantees that $f_{m+1}$ satisfies the second condition of \cref{lem:indicator} from the definition of $\mcD_{\mcX, \gamma}$ and $h$. We prove this via mathematical induction on $\ell$. We first consider the base case, $\ell = 1$. Since $\sigma$ satisfies \cref{cond:step}, there exists a $\sigma$ network $\rho$ of width $1$ such that 
\begin{align*}
|\rho(x)-\step(x)|<\varepsilon
\end{align*}
for all $x\in[0,1]\setminus(-\gamma, \gamma)$ and $\rho([0,1])\subset[0,1]$. Then, we construct a $(\sigma, \iota)$ network $f^{(1)}:[0,1]\to\bbR^2$ of width $2$ as
$$f_{1}(x)_1 = x, \quad f_{1}(x)_2 = \rho(x-z_{m+1}).$$
Then, one can easily observe that $f^{(1)}$ satisfies C1--3. We now consider the general case, $\ell \ge 2$. From the induction hypothesis, for any $\delta>0$, there exists a $(\sigma, \iota)$ network $f_{\ell-1}:[0,1]\to\bbR^2$ of width $2$ such that $f_{\ell-1}(x)_1=x$, $|f_{\ell-1}(x)_2-h_{\ell-1}(x))|<\delta$ and $f_{\ell-1}([0,1])\subset[0,1]^2$. Since $\sigma$ satisfies \cref{cond:step}, for any compact set $\mcC$, there exists a $\sigma$ network $\rho$ of width $1$ such that 
\begin{align*}
|\rho(x)-\step(x)|<\varepsilon/2
\end{align*}
for all $\mcC\setminus(-\gamma, \gamma)$. We now construct $f_{\ell}:[0,1]\to\bbR^2$ as
\begin{align*}
f_{\ell}(x)_1 = f_{\ell-1}(x)_1, \quad f_{\ell}(x)_2 = \rho(f_{\ell-1}(x)_1-z_{m-\ell+2} +(z_{m-\ell+2}-z_{m+\ell})f_{\ell-1}(x)_2)
\end{align*}
Here, by the induction hypothesis, $f_{\ell-1}(x)_1 = x$. Thus, we can simplify this to
\begin{align*}
f_{\ell}(x)_1 = x, \quad f_{\ell}(x)_2 = \rho(x-z_{m-\ell+2} +(z_{m-\ell+2}-z_{m+\ell})f_{\ell-1}(x)_2)
\end{align*}
which is just the substitution of $\step$ in \cref{eq:pflem:indicator1} by $\rho$. Here, one can observe that $f_{\ell}$ satisfies 
C1.
Then, for any $x\in\mcD_{\mcX, \gamma}$
\begin{align}
|f_{\ell}&(x)_2-h_{\ell}(x)|\nonumber\\
\le&|\rho(x-z_{m-\ell+2} +(z_{m-\ell+2}-z_{m+\ell})f_{\ell-1}(x)) -\step(x-z_{m-\ell+2} +(z_{m-\ell+2}-z_{m+\ell})h_{\ell-1}(x))|\nonumber\\
\le&|\rho(x-z_{m-\ell+2} +(z_{m-\ell+2}-z_{m+\ell})f_{\ell-1}(x))-\rho(x-z_{m-\ell+2} +(z_{m-\ell+2}-z_{m+\ell})h_{\ell-1}(x))|\nonumber\\&+
|\rho(x-z_{m-\ell+2} +(z_{m-\ell+2}-z_{m+\ell})h_{\ell-1}(x)) - \step(x-z_{m-\ell+2} +(z_{m-\ell+2}-z_{m+\ell})h_{\ell-1}(x))|.\label{eq:pflem:indicator2}
\end{align}
Here, we note that the second term of \cref{eq:pflem:indicator2} is bounded by $\varepsilon/2$ since 
\begin{align*}
x-z_{m-\ell+2} +(z_{m-\ell+2}-z_{m+\ell})h_{\ell-1}(x)\notin(-\gamma, \gamma)
\end{align*}
for all $x\in\mcD_{\mcX, \gamma}$; since $h_{\ell-1}(x)=0$ or $1$, then $x-z_{m-\ell+2} +(z_{m-\ell+2}-z_{m+\ell})h_{\ell-1}(x) = x-z_{m-\ell+2}$ or $x-z_{m-\ell}$. Hence, we have
\begin{align*}
|f_{\ell}(x)_2
-h_{\ell}(x)|\le \omega_\rho(|(z_{m-\ell+2} - z_{m-\ell})(f_{\ell-1}(x)_2-h_{\ell-1}(x))|)+\varepsilon/2<\omega_\rho(|(z_{m-\ell+2} - z_{m-\ell})\delta|)+\varepsilon/2 < \varepsilon
\end{align*}
by choosing sufficiently small $\delta>0$. It implies that $f_{\ell}$ follows C2. Lastly, we can easily observe that $f_{\ell}(x)_1 = x\in[0,1]$ and $f_{\ell}(x)_2\in[0,1]$ since $\rho(x)\in[0,1]$ for all $x\in\mcC$. It implies that $f_{\ell}$ satisfies C3 and this completes the proof.

\section{Proof of \cref{lem:piecewise-constant}}\label{sec:pflem:piecewise-constant}

In this section, we prove \cref{lem:piecewise-constant}.
To this end, without loss of generality, assume that $\mcK\subset[0,\infty)$ and $c_1, \cdots, c_N\in(0,1)$; if there exists $c_i$ such that $c_i = 0$ or $c_i = 1$, then we can substitute $c_i^* = \varepsilon/2$ or $1-\varepsilon/2$ respectively and approximate them within error $\varepsilon/2$. Let $\xi = \dist{\{c_1, \cdots, c_k\}}{\{0,1\}}$. Then, one can observe that $\xi>0$. In addition, we assume that for any $i\in[N-1]$, $x<y$ for all $x\in \mcI_i$ and $y\in \mcI_{i+1}$. Since $\mcI_i$'s are disjoint, for any $i\in[N-1]$, there exists $x^{(i)}\in\bbR$ such that $\sup\mcI_i<x^{(i)}<\inf\mcI_{i+1}$. Let $x^{(0)} = \min\mcK$, $x^{(N)} = \max\mcK$ and
\begin{align}
    \gamma = \min_{i\in[N-1]}\left\{ \dist{x^{(i)}}{\mcI_i}, \dist{x^{(i)}}{\mcI_{i+1}}\right\}.\label{eq:pflem:piecewise-constant0}
\end{align}
In this proof, we construct a $(\sigma,\iota)$ network $f:\mcK\to[0,1]$ of width $2$ such that for any $k\in[N]$,
\begin{align*}
\sup_{x\in\mcI_k}|f(x)-c_k|\le\eta
\end{align*}
where $\eta:=\min\{\xi, \varepsilon\}$.

To this end, we construct two $(\sigma, \iota)$ networks $h_1:\mcK\to\bbR$ and $h_2:\bbR\to\bbR$ of width $2$ such that 
\begin{itemize}
\item[C1.] for each $k\in[N]$, $\sup_{x\in\mcI_k}|h_1(x)-c_k|\le \eta/2$,
\item[C2.] for any $x\in\bigcup_{k=1}^{N}\mcI_k$, $|h_2\circ h_1(x)-h_1(x)|\le \eta/2$ and $h_2\circ h_1(\mcK)\subset[0,1]$.
\end{itemize}
Then, one can observe that $h_1$ maps input $x$ to near the corresponding $c_k$ if $x\in\mcI_k$, and $h_2$ bounds the codomain of $h_1$ while the approximation for piecewise constant is preserved. If we choose $f = h_2\circ h_1$, then such $f$ satisfies the desired conditions.

We first construct $h_1$ satisfying C1 using the property of $\sigma$ that can approximate $\step$. To this end, we consider a $(\step, \iota)$ network $g:\mcK\to\bbR$ of width $2$ approximating the given piecewise constant function, and then we construct a $(\sigma, \iota)$ network $h_1$ of width $2$ approximating $g$ in $\bigcup_{k=1}^N\mcI_k$. 

We now construct a $(\step, \iota)$ network $g$ approximating piecewise constant function. To construct such $g$, we compose $(\step, \iota)$ networks $g_1, \cdots, g_N:\bbR\to\bbR$ of width $2$ such that each $g_i$ shifts $x$ by a sufficiently large length $L_i>0$ if $x\in[x^{(i-1)}, x^{(i)})$.  Here, for each $i\in[N]$, $L_i$ is defined as $a\times (c_i+b)$ where $a >\max\{1,4x^{(N)}/\eta\}$ and $b = x^{(N)}-\min_{i\in[N]}c_i$ which implies that each $g_i(x) = x+L_i>x^{(N)}$ for $x\in[x^{(i-1)}, x^{(i)})$. i.e., we construct each $g_i$ such that 
\begin{align}
g_i\circ\cdots\circ g_1(x) = 
\begin{cases}
x+a\times(c_1+b)&x\in[x^{(0)}, x^{(1)})\\
x+a\times(c_2+b)&x\in[x^{(1)}, x^{(2)})\\
\vdots\\
x+a\times(c_i+b)&x\in[x^{(i-1)}, x^{(i)})\\
x&\text{otherwise}\label{eq:pflem:piecewise-constant}
\end{cases}
\end{align}
for all $i\in[N]$.
Then, we define $g$ as follows:
$g = g_\mathrm{cut}\circ g_N\circ g_{N-1}\circ\cdots \circ g_1$ where $g_\text{cut}:\bbR\to\bbR$ is defined as 
\begin{align*}
g_\mathrm{cut}(x) = \frac{1}{a}x-b.
\end{align*}
Then, one can easily observe that 
\begin{align*}
g(x) = 
\begin{cases}
c_1+\frac{x}{a}&x\in[x^{(0)}, x^{(1)})\\
c_2+\frac{x}{a}&x\in[x^{(1)}, x^{(2)})\\
\vdots\\
c_N+\frac{x}{a}&x\in[x^{(N-1)}, x^{(N)}].
\end{cases}
\end{align*}
Since $a>4x^{(N)}/\eta$, it holds that $|x/a|<\eta/4$ for all $x\in\mcK$. Thus, $g$ approximates the piecewise constant function within an error $\eta/4$.

We now construct $(\step, \iota)$ networks $g_1, \cdots, g_N$ satisfying \cref{eq:pflem:piecewise-constant}.
For each $i\in[N]$, we define $g_i:\bbR\to\bbR$ as
\begin{align*}
g_i(x) = x+a\times(c_i+b)\step(-(x-x^{(i)})).
\end{align*}
One can observe that $g_i$ shifts $x$ by $a\times(c_i+b)$ if $x<x^{(i)}$. Here, we note that since $a\times(c_i+b)>x^{(N)}$, the values shifted by $g_i$ for some $i\in[N]$ are not shifted again, resulting that $g_i$ shifts only $x\in[x^{(i-1)}, x^{(i)})$. Thus, our $g$ can approximate a given piecewise function within an error $\eta/2$.

We now construct a $(\sigma, \iota)$ network $h_1$ of width $2$ approximating $g$ on $\bigcup_{i\in[N]}\mcI_i$. Since $\sigma$ is squashable, then for any compact set $\mcC$ and $\alpha>0$, there exists a $\sigma$ network $\rho:\bbR\to\bbR$ such that $\rho$ is increasing on $\mcC$, $\rho(\mcC)\subset[0,1]$, and
\begin{align*}
|\rho(x)-\step(x)|<\alpha
\end{align*}
for all $x\in\mcC\setminus(-\beta, \beta)$ where $0<\beta<\min\{\gamma, \xi\}$ (\cref{eq:pflem:piecewise-constant0}). We will give an explicit value to $\alpha$ later. We now construct a $(\sigma, \iota)$ network $h_1$ of width $2$ as follows:
\begin{align*}
h_1 &= g_\mathrm{cut}\circ f_N\circ\cdots\circ f_1 \text{ where}\\
f_i(x)&= x+a\times(c_i+b)\rho(-(x-x^{(i)}))\quad\forall i\in[N].
\end{align*}
Then, one can observe that $|f_i(x)-g_i(x)| = |c_i+b||\rho(-(x-x^{(i)})) - \step(-(x-x^{(i)}))|<|c_i+b|\alpha$ for all $x\in\mcC\setminus\mcB_\beta(x^{(i)})$ and $i\in[N]$. For the notational simplicity, we denote $\delta_i = |c_i+b|\alpha$. We note that for any $i,j\in[N]$ and $x\in\mcI_i$, 
\begin{align}
g_{j}\circ\cdots\circ g_1(x)\notin\mcB_\beta(x^{(i)}).\label{eq:pflem:piecewise-constant2}
\end{align}
\cref{eq:pflem:piecewise-constant2} holds since $g_{j}\circ\cdots\circ g_1$ maps $x$ to $x\in\mcI_i$, or a value out of $\mcK$ ($x+a\times(c_i+b)>x+x^{(N)}$ from the definition of $a$ and $b$) and $\beta<\gamma$. Then, for any $i\in[N]$ and $x\in\mcI_i$, it holds that 
\begin{align*}
|h_1(x) - g(x)|&=|g_\mathrm{cut}\circ f_N\circ\cdots\circ f_1(x) -  g_\mathrm{cut}\circ g_N\circ\cdots\circ g_1(x)|\\
&\le\omega_{g_\mathrm{cut}}(|f_N\circ\cdots\circ f_1(x) - g_N\circ\cdots\circ g_1(x)|)\\
&\le\omega_{g_\mathrm{cut}}(|f_N\circ\cdots\circ f_1(x) - f_N\circ g_{N-1}\circ\cdots\circ g_1(x)+f_N\circ g_{N-1}\circ\cdots\circ g_1(x) - g_N\circ\cdots\circ g_1(x)|)\\
&\le \omega_{g_\mathrm{cut}}(\omega_{f_N}(|f_{N-1}\circ\cdots\circ f_1(x) - g_{N-1}\circ\cdots\circ g_1(x)|)+|f_N\circ g_{N-1}\circ\cdots\circ g_1(x) - g_N\circ\cdots\circ g_1(x)|)
\end{align*}
Here, $|f_N\circ g_{N-1}\circ\cdots\circ g_1(x) - g_N\circ\cdots\circ g_1(x)|<\delta_N$ from \cref{eq:pflem:piecewise-constant2}. Thus, by conducting this procedure iteratively, we have
\begin{align*}
|h_1(x)-g(x)|&\le|\omega_{g_\mathrm{cut}}(\omega_{f_N}(f_{N-1}\circ\cdots\circ f_1(x) - g_{N-1}\circ\cdots\circ g_1(x))+\delta_N)|\\
&\le|\omega_{g_\mathrm{cut}}(\omega_{f_N}(\omega_{f_{N-1}}(f_{N-2}\circ\cdots\circ f_1(x) - g_{N-2}\circ\cdots\circ g_1(x))+\delta_{N-1})+\delta_N)|\\
&\vdots\\
&\le|\omega_{g_\mathrm{cut}}(\omega_{f_N}(\cdots(\omega_{f_2}(\delta_1)+\delta_2)\cdots)+\delta_N)|<\eta/4
\end{align*}
by choosing sufficiently small $\alpha>0$, which leads us to have sufficiently small $\delta_i$ for all $i\in[N]$. Consequently, for any $i\in[N]$ and $x\in\mcI_i$, we have
\begin{align}
|h_1(x)-c_i| \le |h_1(x)-g(x)|+|g(x)-c_i|<\eta/4+\eta/4 = \eta/2.
\end{align}
Hence, our $h_1$ satisfies C1.

We now construct $h_2$ satisfying C2. We suppose that there exists $u\in\mcK$ such that $h_1(u)<0$; we will discuss the case that there exists $v\in\mcK$ such that $h_1(v)>1$ later.
To this end, we consider a $(\sigma, \iota)$ network of width 2 that iteratively adds some constant to the region such that $h_1(x)<0$. 
Namely, it suffices to show that for any $\varepsilon'>0$, there exists a $(\sigma, \iota)$ network $\psi:\bbR\to\bbR$ of width $2$ such that 
\begin{itemize}
\item if $x\ge \eta/4$, then $|\psi(x) - x|\le\varepsilon'$
\item if $x\in(0,\eta/4)$, then $\psi(x)\in[0,1]$, and
\item if $x\le0$, then $\psi(x)-x\ge 1/2$.
\end{itemize}
Then, let $h_2 = \psi^{N_1}$ for some $N_1\in\bbN$ such that $N_1/2>|\min_{x\in\mcK} h_1(x)|$, then we obtain 
\begin{align*}
h_2\circ h_1(x) = \psi^{N_1}\circ h_1(x)\in[0,1]
\end{align*}
for all $x\in\mcK$ such that $h_1(x)<0$.

Furthermore, since $\eta\le\xi$ and $h_1$ satisfies C1, $h_1(x)\ge \eta/2$ for any $x\in\bigcup_{i=1}^N\mcI_i$. Thus, if we choose sufficiently small $\varepsilon'>0$ such that $\varepsilon'<\eta/(4N_1)$, then
\begin{align*}
|h_2\circ h_1(x)-h_1(x)| &= |\psi^{N_1}\circ h_1(x)-h_1(x)|
\\&\le|\psi^{N_1}\circ h_1(x)-\psi^{N_1-1}\circ h_1(x)|+\cdots+|\psi\circ h_1(x)-h_1(x)|\\
&\le N_1\varepsilon'\le \eta/4\le\eta/2.
\end{align*}
Here, for each $i\in[N_1-1]$, $\psi^{i}\circ h_1(x)\ge \eta/4$ since $\psi^{i}\circ h_1(x)\subset\mcB_{i\varepsilon'}(h_1(x))$ and $h_1(x)\ge \eta/2$. It guarantees that $|\psi^{i}\circ h_1(x)-\psi^{i-1}\circ h_1(x)|\le \varepsilon'$ for each $i\in[N_1]$. Hence, $h_2$ satisfies C2.  If there exists $v\in\mcK$ such that $h_1(v)>1$, then the same argument can be applied with the choice of $\psi_1(x)  = 1-\psi(1-x)$.

We now construct such $\psi$ using the property that $\sigma$ network can approximate $\step$.
Since $\sigma$ is squashable, for any $\delta'>0$ and a compact set $\mcD$, there exists a $\sigma$ network $\rho^\ast:\mcD\to\bbR$ such that 
\begin{align*}
|\rho^\ast(x)-\step(x)|<\delta'
\end{align*}
for all $x\in\mcD\setminus(-\eta/8, \eta/8)$.
We choose $\delta'>0$ such that $\delta'\le \min\{3\varepsilon'/2, 1/4\}$. Consider a $(\sigma, \iota)$ network $\psi$ of width 2 defined as 
\begin{align*}
\psi(x) = x+\frac{2}{3}\rho^\ast(-(x-\eta/8)).
\end{align*}
Then, one can easily observe that $|\psi(x)-x|<2\delta'/3\le\varepsilon'$ if $x\ge \eta/4$, $\psi(x)\in(0,\eta/4+2/3)\subset[0,1]$ if $x\in(0, \eta/4)$, and $|\psi(x)-(x+2/3)|<2\delta'/3\le 1/6$ if $x\le 0$ which implies $\psi(x)-x\ge 1/2$. It completes the proof.

\newpage

\section{Proof of \cref{lem:inductive-encoder}}\label{sec:pflem:inductive-encoder}
In this section, we prove \cref{lem:inductive-encoder}. To this end, we construct a $(\sigma, \iota)$ network $f$ of width $2$ that maps for each $\mcS\in\mcG$ to a disjoint interval. Then, since $f$ is continuous, $\{f(\mcS):\mcS\in\mcG\}$ is a $1$-grid of size $n_1n_2$ and this completes the proof.
Before we illustrate our proof, we define the additional notation used in this proof. Since $\mcG$ is a $2$-grid of size $(n_1,n_2)$, there exist compact intervals $[a_1, b_1], \cdots, [a_{n_1}, b_{n_1}], [a_1', b_1'], \cdots, [a_{n_2}', b_{n_2}']$ satisfying the following:
\begin{itemize}
\item $b_i<a_{i+1}$ and $b_j'<a_{j+1}'$ for each $i\in[n_1-1]$ and $j\in[n_2-1]$, respectively,
\item for any $\mcS\in\mcG$, there uniquely exist $i\in[n_1]$ and $j\in[n_2]$ such that $\mcS = [a_i, b_i]\times[a_j', b_j']$.
\end{itemize}
For each $i\in[n_1]$ and $j\in[n_2]$, let $\mcU_{ij} = [a_i, b_i]\times[a_j', b_j']$,  $\mcV_{i} = \bigcup_j\mcU_{ij}$,
\begin{align*}
\eta = \min_{j\in[m-1]}\left\{a_{j+1}' - b_{j}'\right\}
\end{align*}
and $L = b_m' - a_1'$. We write $e_1 = (1,0)$ and $e_2 = (0,1)\in\bbR^2$. For $i\in\{1,2\}$ and $b\in\bbR$, we use $\mcH(e_i, b) \defeq \{x\in\bbR^2|x_i+b=0\}$. 
We first consider a $(\sigma, \iota)$ network $h_1:\mcK\to\bbR^2$ of width $2$ defined as
\begin{align*}
h_1(x)_1 = \rho(x_1 - c_1), \quad h_1(x)_2 = \iota(x_2)
\end{align*}
where $c_1\in(b_{n-1}, a_n)$ and $\rho$ is a $\sigma$ network of width $1$ such that $|\rho(x)-\step(x)|<\zeta$ on $[a_1, b_n]$. We will assign an explicit value to $\zeta$. Then, one can observe that 
\begin{align}
h_1(\mcV_{n})\subset\mcB_\zeta(\mcH(e_1, -1)) ,\quad h_1(\mcV_{i})\subset\mcB_\zeta(\mcH(e_1, 0)) \text{ for all }i\in[n-1].\label{eq:pflem:inductive-encoder}
\end{align}
Furthermore, since $h_1(x)_1$ is strictly increasing on $\mcK$, the ordering of $\mcV_i$'s with respect to the first coordinate is preserved: if $i<j$, then $x_1<y_1$ for all $x\in\mcV_i$, $y\in\mcV_j$. We then iteratively apply some $(\sigma, \iota)$ networks $h_2, \cdots, h_{n_1}$ so that for each $i\in[n_1]$, $h_i$ maps $\mcV_{{n_1}-i+1}$ to $\mcB_\zeta(\mcH(e_1, -1))$ and shifts $\mcV_{n_1-i+1}$ by sufficiently large length such that the images of $\mcV_i$ are disjoint for the second coordinate. 

We now formally construct such $(\sigma, \iota)$ networks $h_2, \cdots, h_{n_1}$. See the following lemma where the proof is deferred to \cref{sec:pflem:separating}.

\begin{lemma}\label{lem:separating}
Let $\xi>0$ and $r>0$.
Let $\mcX_0\subset\mcB_\xi(\mcH(e_1,0))$, $\mcX_1\subset\mcB_\xi(\mcH(e_1,-1))$, and $\mcY\subset\mcB_\xi(\mcH(e_1,0))$ be compact sets in $\bbR^2$ such that $y_1>x_1$ for all $x\in\mcX_0$ and $y\in\mcY$.
Then, there exists a $(\sigma,\iota)$ network $f:\bbR^2\to\bbR^2$ of width $2$ satisfying the following properties:
\begin{itemize}[leftmargin=10pt]
\item for any $x\in\mcX_0\cup\mcX_1$, $|f(x)_2 - x_2|<2r\xi$,
\item for any $y\in\mcY$, $|f(y)_2-(y_2+r)|<2r\xi$,
\item $f(\mcX_0)\subset\mcB_\xi(\mcH(e_1,0))$ and $f(\mcY),f(\mcX_1)\subset\mcB_\xi(\mcH(e_1,-1))$,
\item there exists strictly increasing $\phi:\bbR\to\bbR$ such that $f(x)_1=\phi(x_1)$ for all $x\in\mcX_0$.
\end{itemize}
\end{lemma}

\cref{lem:separating} implies that there exists a $(\sigma, \iota)$ network of width $2$ that maps $\mcY$ to in $\mcH(e_1, -1)$ with approximately shift for the second coordinate by $r$. From \cref{eq:pflem:inductive-encoder}, we can apply \cref{lem:separating} with
\begin{align*}
\mcX_0 = \bigcup_{i\in[n_1-2]}h_1(\mcV_{i}), \quad\mcX_1 = h_1(\mcV_{n_1}), \quad\mcY = h_1(\mcV_{n_1-1}),
\end{align*}
$r = L+1$ and $\xi = \zeta$. Then, there exists a $(\sigma, \iota)$ network $h_2$ of width $2$ that maps the points of $\mcX_0$ and $\mcX_1$ approximately identically while shifting the second coordinate of $\mcY$ by $L+1$. Here, one can observe that if we choose a sufficiently small $\zeta>0$, then $h_2(h_1(\mcV_{n_1}))$ and $h_2(h_1(\mcV_{n_1-1}))$ are disjoint for the second coordinate by our choice of $r$.
Furthermore, from the third and fourth lines of the properties listed in \cref{lem:separating}, \cref{lem:separating} can be applied iteratively with the recursive choice of $\mcX_0, \mcX_1, \mcY,r$ and $\xi$ in \cref{lem:separating}. In particular, by the fourth line of the properties from \cref{lem:separating}, the ordering of $\mcV_i$'s with respect to the first coordinate is preserved while \cref{lem:separating} is applied. Thus, among the sets contained in $\mcX_0$, we can choose $\mcY$ as the set that is the highest with respect to the first coordinate.

We now construct such $(\sigma, \iota)$ networks $h_2, \cdots, h_{n_1}:\bbR^2\to\bbR^2$ as follows: for each $k\in[n_1]\setminus\{1\}$, $h_k$ is from \cref{lem:separating} with the choices of 
\begin{itemize}
\item $\mcX_0 = \bigcup_{i\in[n_1-k]}h_{k-1}\circ\cdots\circ h_1(\mcV_{i})$,
\item $\mcX_1 = \bigcup_{i\in[k-1]}h_{k-1}\circ\cdots\circ h_1(\mcV_{n_1-i+1})$,
\item $\mcY = h_{k-1}\circ\cdots\circ h_1(\mcV_{n_1-k+1})$,
\item $r = r_k$ where $r_k = (k-1)(L+1)$ and $\xi = \zeta$.
\end{itemize}
Then, we construct a $(\sigma, \iota)$ network $f:\mcK\to\bbR$ of width $2$ as
\begin{align}
f(x) = p\circ h_{n_1}\circ\cdots\circ h_1(x)\label{eq:pflem:inductive-encoder1}
\end{align}
where $p:\bbR^2\to\bbR$ is a projection onto the second coordinate: $p(x,y) = y$. 
We now prove that if we choose sufficiently small $\zeta>0$ such that 
\begin{align*}
\sum_{k=2}^{n_1} 2r_k\zeta<\min\left\{\frac{\eta}{2}, \frac{1}{2}\right\},
\end{align*}
then for each $i\in[n_1]$ and $j\in[n_2]$, $f(\mcU_{ij})$ is disjoint.

We first show that if $i,j\in[n_1]$ such that $i< j$, then $f(x)>f(y)$ for all $x\in\mcV_i$ and $y\in\mcV_j$, and then we prove that for each $i\in[n_1]$, if $j,j'\in[n_2]$ such that $j<j'$, then $f(x)<f(y)$ for all $x\in\mcU_{ij}$ and $y\in\mcU_{ij'}$.

We first consider $x\in\mcV_i$ and $y\in\mcV_j$. From our definition of $f$ (\cref{eq:pflem:inductive-encoder1}) and \cref{lem:separating}, one can observe that 
\begin{align*}
|f(x) - (x_2+r_{n_1-i+1})|<\sum_{k=2}^{n_1} 2r_k\zeta\le\frac{1}{2}, \quad |f(y) - (y_2+r_{n_1-j+1})|<\sum_{k=2}^{n_1} 2r_k\zeta\le\frac{1}{2}.
\end{align*}
Since $r_{n_1-i+1}-r_{n_1-j+1}\ge L+1$, 
the above equation implies that 
\begin{align*}
f(x) - f(y)>r_{n_1-i+1}-r_{n_1-j+1}- (y_2-x_2)-1\ge L+1 - L -1=0
\end{align*}

We now consider $x\in\mcU_{ij}$ and $y\in\mcU_{ij'}$. As in above, we have
\begin{align*}
|f(x) - (x_2+r_{n_1-i+1})|<\sum_{k=2}^{n_1} 2r_k\zeta\le\frac{\eta}{2},\quad
|f(y) - (y_2+r_{n_1-i+1})|<\sum_{k=2}^{n_1} 2r_k\zeta\le\frac{\eta}{2}.
\end{align*}
Since $y_2-x_2>\eta$ by the definition of $\eta$, we have 
\begin{align*}
f(y)-f(x) > y_2-x_2-\eta>0
\end{align*}
and this completes the proof.

\subsection{Proof of \cref{lem:separating}}\label{sec:pflem:separating}

In this section, we prove \cref{lem:separating}. Let $b\in\bbR$ such that $x_1<b<y_1$ for all $x=(x_1, x_2)\in\mcX_0$ and $y = (y_1,y_2)\in\mcY$ and
\begin{align*}
\eta = \min\{y_1-b, b-x_1|y\in\mcY, x\in\mcX_0\}.
\end{align*}
We note that such $b$ is well-defined and $\eta>0$ because $x_1<y_1$ for all $x\in\mcX_0$, $y\in\mcY$ and $\mcX_0, \mcY$ are compact.
Since $\sigma$ is squashable, for any compact set $\mcK$, there exists a $\sigma$ network $\rho:\bbR\to\bbR$ of width $1$ such that 
\begin{align*}
|\rho(x) - \step(x)|<\xi
\end{align*}
for all $x\in\mcK\setminus(-\eta, \eta)$.
Let 
$A = \begin{bmatrix}
1&0\\-r&1
\end{bmatrix}$. Then, one can easily observe that
$A^{-1} = 
\begin{bmatrix}
1&0\\r&1
\end{bmatrix}$.

We now define functions $f_1,f_2,f_3:\bbR^2\to\bbR^2$ as
\begin{align*}
f_1(x) = Ax, \quad f_2(x) = (\rho(x_1-b), \iota(x_2)),\quad f_3(x) = A^{-1}x    
\end{align*}
for all $x = (x_1, \cdots, x_n)\in\bbR^n$, respectively. We now define a function $f:\bbR^n\to\bbR^n$ as
\begin{align*}
    f(x) = (f_3\circ f_2\circ f_1)(x)
\end{align*}
for all $x\in\bbR^n$. Then, $f$ is a $(\sigma, \iota)$ network of width $2$ and
\begin{align}
f(x)_1 = \rho(x_1-b),\quad f(x)_2 = x_2+r(\rho(x_1-b)-x_1).\label{eq:pflem:separating}
\end{align}

We now show that our $f$ satisfies the properties listed in \cref{lem:separating}. One can easily observe that $f$ satisfies the fourth property of \cref{lem:separating}. Thus, we consider the first--third properties. From \cref{eq:pflem:separating}, we can classify the image regions corresponding to each input region.

We first consider $x\in\mcX_0$. Since $\mcX_0\subset\mcB_\xi(\mcH(e_1, 0))$ and $x_1<b-\eta$, we have $x_1\in(-\xi, \xi)$ and $\rho(x_1 - b)\in(0,\xi)$. Thus, it holds that $f(x)_1\in\mcB_\xi(\mcH(e_1, -1))$ and $|f(x)_2-x_2|<2r\xi$. We now consider $x\in\mcX_1$. Since $\mcX_1\subset\mcB_\xi(\mcH(e_1, -1))$ and $x_1>b+\eta$, we have $x_1\in(1-\xi, 1+\xi)$ and $\rho(x_1 - b)\in(1-\xi,1)$. Thus, $f(x)_1\in\mcB_\xi(\mcH(e_1, 0))$ and $|f(x)_2-x_2|<2r\xi$. Lastly, let $y\in\mcY$. Since $\mcY_1\subset\mcB_\xi(\mcH(e_1, 0))$ and $y_1>b+\eta$, we have $y_1\subset(-\xi, \xi)$ and $\rho(y_1-b)\in(1-\xi, 1)$. Thus, $f(y)_1\subset\mcB_\xi(\mcH(e_1,-1))$ and $|f(y)_2-(y_2+r)|<2r\xi$. Conclusively, $f$ satisfies all properties listed in \cref{lem:separating} and this completes the proof.

\end{document}